%% file: main.tex
\documentclass[sigconf, nonacm]{acmart}

\usepackage{xspace}
\usepackage{multirow}
\usepackage{subfig}

\usepackage{xpatch}
\usepackage[frozencache,cachedir=./pyg/]{minted}
\xpatchcmd{\minted}{\VerbatimEnvironment}{\VerbatimEnvironment\let\itshape\relax}{}{}

\newcommand\vldbpagestyle{plain} 

\newtheorem{theorem}{Theorem}[section]
\newtheorem{lemma}{Lemma}
\newtheorem{assumption}[theorem]{Assumption}
\newtheorem{proposition}{Proposition}
\newcommand{\ours}{{\sf FederatedScope}\xspace}

\definecolor{keyword_color}{RGB}{40,40,255}
\definecolor{comment_color}{RGB}{0,96,96}

\newcommand{\stitle}[1]{\vspace{0.6ex}\noindent{\bf #1}}

\newcommand{\ssstitle}[1]{\vspace{0.6ex}\noindent$\bullet$~{\em #1}}

\newcommand{\squishlist}{
	\begin{list}{$\bullet$}{
		\setlength{\itemsep}{0pt}
		\setlength{\parsep}{3pt}
		\setlength{\topsep}{3pt}
		\setlength{\partopsep}{0pt}
		\setlength{\leftmargin}{1.0em}
		\setlength{\labelwidth}{1em}
		\setlength{\labelsep}{0.5em}
   }
}

\newcommand{\squishenum}{
	
	\begin{list}{\usecounter{scount}}{
		\setlength{\itemsep}{0pt}
		\setlength{\parsep}{3pt}
		\setlength{\topsep}{3pt}
		\setlength{\partopsep}{0pt}
		\setlength{\leftmargin}{1.2em}
		\setlength{\labelwidth}{1em}
		\setlength{\labelsep}{0.5em}
	}
}

\newcommand{\squishend}{
	\end{list}
}

\begin{document}
\title{FederatedScope: A Flexible Federated Learning Platform \\ for Heterogeneity}

\author{
Yuexiang Xie$^*$,
Zhen Wang$^*$, 
Dawei Gao,
Daoyuan Chen$^\dagger$,
Liuyi Yao$^\dagger$,
Weirui Kuang$^\dagger$, \\
Yaliang Li$^\ddagger$,
Bolin Ding$^\ddagger$,
Jingren Zhou \\ 
Alibaba Group
}

\begin{abstract} 
Although remarkable progress has been made by existing federated learning (FL) platforms to provide infrastructures for development, these platforms may not well tackle the challenges brought by various types of heterogeneity, including the heterogeneity in participants' local data, resources, behaviors and learning goals.
To fill this gap, in this paper, we propose a novel FL platform, named \emph{\ours}, which employs an event-driven architecture to provide users with great flexibility to independently describe the behaviors of different participants.
Such a design makes it easy for users to describe participants with various local training processes, learning goals and backends, and coordinate them into an FL course with synchronous or asynchronous training strategies.
Towards an easy-to-use and flexible platform, \ours enables rich types of plug-in operations and components for efficient further development, and we have implemented several important components to better help users with privacy protection, attack simulation and auto-tuning.
We have released \ours at \href{https://github.com/alibaba/FederatedScope}{https://github.com/alibaba/FederatedScope} to promote academic research and industrial deployment of federated learning in a wide range of scenarios.
\end{abstract}

\maketitle

\pagestyle{\vldbpagestyle}

\renewcommand*{\thefootnote}{\fnsymbol{footnote}}
\footnotetext[1]{Co-first authors.} 
\footnotetext[2]{Equal contribution, listed in alphabetical order.} 
\footnotetext[3]{Corresponding authors, email addresses: \{yaliang.li, bolin.ding\}@alibaba-inc.com}
\renewcommand*{\thefootnote}{\arabic{footnote}}
\renewcommand{\shortauthors}{Yuexiang Xie, Zhen Wang, Dawei Gao, Daoyuan Chen, Liuyi Yao, Weirui Kuang, Yaliang Li, Bolin Ding, and Jingren Zhou}
\renewcommand{\shorttitle}{FederatedScope: A Flexible Federated Learning Platform for Heterogeneity}

\input{subsections/1_introduction}

\input{subsections/2_preliminary}
\input{subsections/3_event_driven}
\input{subsections/4_plugins}

\input{subsections/5_exp}
\input{subsections/6_conclusion}

\bibliographystyle{ACM-Reference-Format}
\bibliography{FederatedScope}

\clearpage
\appendix
\input{subsections/Appendix}

\end{document}

%% file: subsections/1_introduction.tex
\section{Introduction}
\label{sec:intro}
As one of the feasible solutions to address the privacy leakage issue when utilizing isolated data from multiple sources, Federated Learning (FL)~\cite{mcmahan2017communication,konevcny2016federated,yang2019federated} has rapidly gained enormous popularity in both academia and industry~\cite{yang2019federated_app,hard2018federated,xu2021federated, leroy2019federated}. Such widespread adoption of FL is inextricably tied to the support of FL platforms, such as TFF~\cite{tff}, FATE~\cite{yang2019federated}, PySyft~\cite{pysyft} and FedML~\cite{fedml}, which provide users with functionalities to get started quickly and focus on developing new FL algorithms and applications.

Although the existing FL platforms have made remarkable progress, there are more burgeoning demands from FL research and deployment, which are mainly brought by the heterogeneity of FL.
Specifically, we summarize the heterogeneity of FL as the following four aspects.

(1) \textbf{Heterogeneity in Local Data}. 
The isolated data in FL, which are often owned by different parties or generated by different edge devices, vary a lot among the FL participants in terms of quality, quantity, underlying distributions, etc.
Such heterogeneity in data can lead to the sub-optimal performance when applying the vanilla FedAvg~\cite{mcmahan2017communication}, i.e., producing one global model for all the participants by the same local training process.
Recent studies on Personalized FL~\cite{fallah2020personalized,tan2021towards} customize the local training processes of the participants according to their local data, including applying client-specific parameters, training schedules, submodules, and fusing approaches. Supporting these customizations brings challenges to the extensibility and flexibility of FL platforms.

(2) \textbf{Heterogeneity in Participants' Resources}. 
Apart from the heterogeneity of data, the participants' resources can also be very different, including computation resources, storage resources, communication bandwidths, reliability, and so on.
However, most of the existing FL platforms~\cite{tff, yang2019federated, fedml} adopt a synchronous training strategy for FL, which might lead to additional overhead caused by the heterogeneity in participants' resources.
For example, with the synchronous training strategy, the whole FL system could usually suffer from the slow clients that might be caused by network congestion, sluggish local training, or even device crash.
Thus, it would be better if FL platforms allow users to implement/execute FL with asynchronous training strategies~\cite{chen2020asynchronous,wu2020safa} to ensure both efficiency and effectiveness in real-world FL applications.

(3) \textbf{Heterogeneity in Participants' Behaviors}.
In the vanilla FedAvg~\cite{mcmahan2017communication}, the participants only exchange homogeneous information (e.g., model parameters or gradients) and have the same behaviors (e.g., updating models based on the local data via SGD).
However, the practical and recent FL applications often require to exchange various types of information among participants and execute diverse training processes, which leads to rich behaviors.
For example, in the emerging federated graph learning applications~\cite{xie2021federated,zhang2021subgraph,wu2021fedgnn,wang2022federatedscopegnn}, participants exchange and handle multiple types of information, including model parameters, gradients, node embeddings, etc; Real-world FL applications might involve participants with different backends, which indicates that participants need backend-dependent computation graphs and accordingly perform diverse training processes. 
Besides, due to the heterogeneity in local data discussed above, participants could locally train with client-specific configurations that are suitable for their local data to achieve a better model utility.
The heterogeneity in participants' behaviors, caused by handling various exchanged information, executing diverse training processes and applying client-specific configurations in local training, prompts the FL platforms to support flexible expression for rich behaviors of participants.

(4) \textbf{Heterogeneity in Learning Goals}.
Towards a more general scope of utilizing isolated data, some recent FL studies~\cite{xie2021federated,fedem,smith2017federated} propose to allow participants to collaboratively learn common knowledge while optimizing for different learning goals.
For example, several research institutes can federally train a graph neural network to capture the generalizable structural patterns of molecules by using their own molecule data that correspond to different learning goals, such as predicting solubility, enzyme type, penetration, etc.
Another example, which can benefit from such heterogeneous federated learning, is the pre-trained language models in natural language processing~\cite{tian2022fedbert}, since participants can collaboratively train with different objectives based on their private corpora that cannot be directly shared in some scenarios such as medicine and finance.
To handle the heterogeneity in learning goals, FL platforms should allow participants to locally train with different learning objectives and only share parts of the local models for collaboratively learning.

The aforementioned aspects of heterogeneity are commonly observed in real-world FL applications. Although we discuss them in the above four aspects, they can appear jointly in a single application. 
Facing such mixed heterogeneity, users are eager for an FL platform that is \textbf{flexible}: Participants should be allowed to express their diverse behaviors and different learning goals according to their own local data and system resources, and these participants can be effortlessly coordinated with synchronous or asynchronous training strategies for completing the federal training procedure based on a pre-defined consensus.
Towards such a purpose, in this paper, we propose \ours, a novel FL platform to handle the heterogeneity of FL.

To provide such flexibility, we propose \ours, a novel FL platform that employs the event-driven architecture~\cite{eventdriven, kafka} to frame FL courses into $<$event, handler$>$ pairs. 
Note that it is not trivial to build up a comprehensive FL platform with such a formalization. In particular, considering the heterogeneity of federated learning, such formalization is expected to express diverse behaviors of servers and clients for handling the heterogeneity, and be well-modularized so that users can conveniently develop new FL algorithms and applications.
To fulfill this goal, the events provided in \ours can be categorized into two types, i.e., \textit{events related to message passing} and \textit{events related to condition checking}, which are used to describe what happens in the FL courses from the perspective of an individual participant.
For example, in a vanilla FedAvg, a typical event of clients is ``{\sf receiving\_models}'', which indicates that clients receive the global model from the server, and the corresponding handler can be ``\textit{train the received global model based on the local data, and then return the model updates}''.
The handlers, triggered by events, describe what actions should be taken when a specific event happens.
These events happen in the intended logical order and naturally trigger the corresponding handlers, which can precisely express various FL algorithms and procedures.
All the participants can be coordinated with the pre-defined events related to message passing and condition checking to construct suitable FL course for specific scenarios and applications.

\begin{figure}[t]
    \centering
    \includegraphics[width=0.45\textwidth]{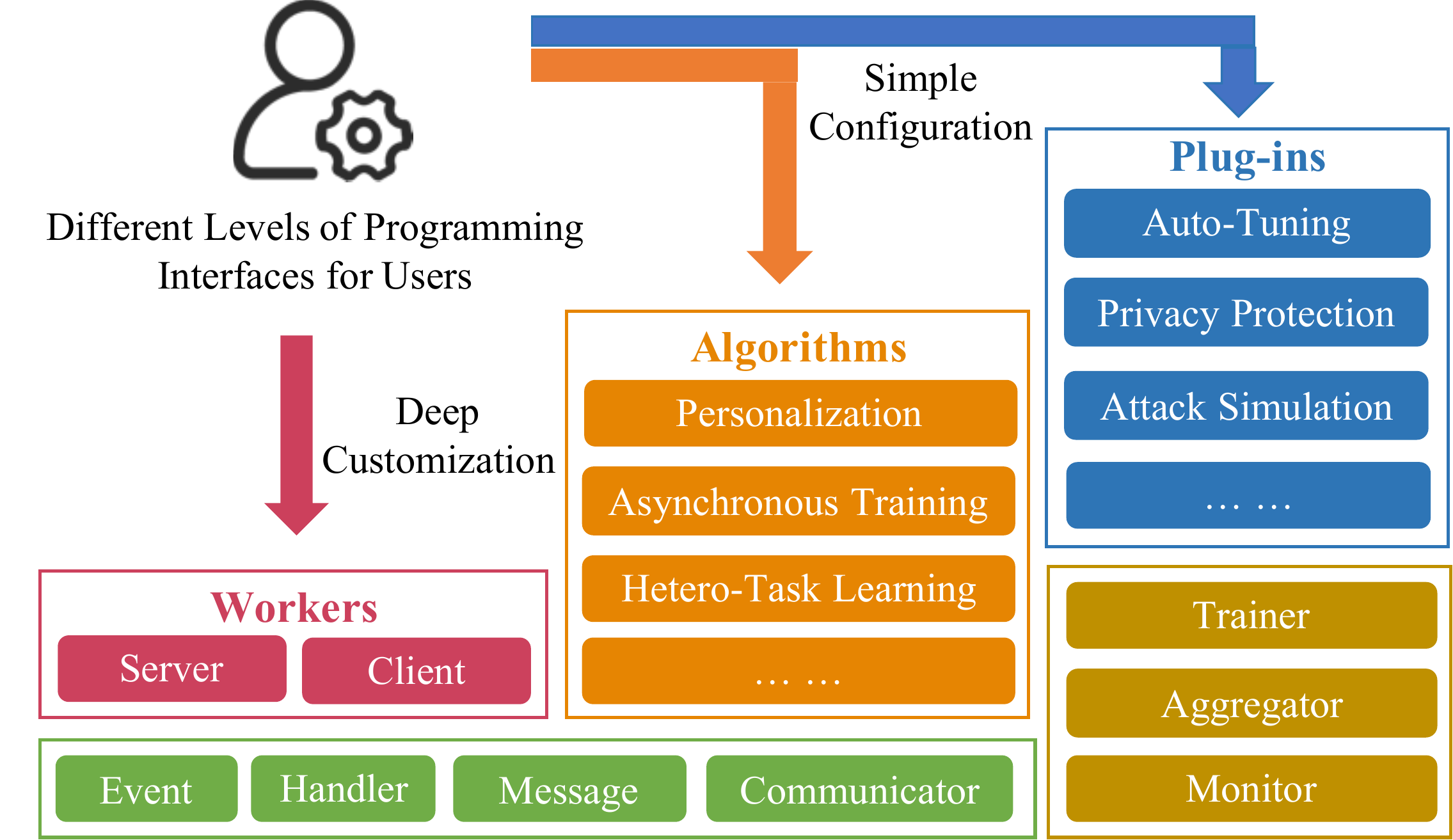}
    \vspace{-0.10in}
    \caption{\ours provides different levels of programming interfaces for users.}
    \vspace{-0.10in}
    \label{fig:programming_interfaces}
\end{figure}

Besides the flexibility, as an FL platform, \ours also provide great \textbf{usability}, that is, \ours provides different levels of programming interfaces to meet different requirements from users, as demonstrated in Figure~\ref{fig:programming_interfaces}.  
For the users who want to design new FL algorithms, as discussed above, \ours allows them to add new $<$event, handler$>$ pairs to implement their ideas.
For the users who want to directly apply existing FL techniques to certain application scenarios, \ours provides rich sets of events and corresponding handlers, core functionalities and several important plug-in components, all of which can be directly called, and thus users only need to focus on a necessary set of interfaces to be integrated or implemented. 
For example, we have implemented several personalization federated algorithms, including applying client-wise configuration, maintaining client-wise sub-modules, global-local fusing, etc, for users' convenient usage. 
Besides personalization, \ours provides users with functionalities such as asynchronous training, privacy protection and cross-backend FL, and several important plug-ins such as attack simulation for protection verification, and auto-tuning for helping users to automatically seek suitable hyperparameters.

Last but not least, \ours also has great \textbf{extensibility}, which is brought by the fact that the set of $<$event, handler$>$ pairs can be easily extended by adding new ones. 
Take personalization again as an example: to add a new personalization, users only need to add new behaviors (e.g., adopting client-specific training course) in the corresponding handlers. 
Such extension convenience also holds for all the other functions such as federated aggregators, asynchronous training, privacy-protection, etc. 
Through this way, \ours can be easily extended to include new functions or plug-ins to satisfy new requirements brought by new developments and support a variety of new scenarios.

\stitle{Contributions.}
Our contributions can be summarized as: 
(1) Motivated by the heterogeneity challenges from a wide range of FL applications, we propose and release \ours, a novel FL platform to handle heterogeneity in FL. The proposed \ours promotes the development of FL techniques and the deployment of FL applications. 
(2) With the event-driven architecture, \ours provides users with rich yet extendable sets of events and corresponding handlers, core functionalities such as asynchronous training, personalization and cross-backend FL, and several important plug-in components. These implementations make it easy for users to apply FL algorithms in both academia and industry applications.
(3) \ours brings great flexibility, usability and extensibility to users, broadens the application scope and enables more tasks that would otherwise be infeasible due to challenges brought by various types of heterogeneity in FL.

%% file: subsections/2_preliminary.tex
\section{Preliminary}
\subsection{Problem Definition}
\label{sec:preliminary}

Federated Learning (FL)~\cite{konevcny2016federated, mcmahan2017communication, yang2019federated}, a learning paradigm for collaboratively training models from dispersed data without directly sharing private information, involves multiple participants who are willing to contribute their local data and computation resources.
We use \textit{server} to denote the participant(s) who are responsible for coordinating and aggregating, while other participants are \textit{clients}.
During a typical \textit{training round} of an FL course, clients update the global model received from the server by training it with local data, and send the model updates back to the server for collaborative aggregation. In repeated training rounds, the (possibly sensitive) training data is always kept locally in each client; the server and clients only exchange aggregated and meta information, such as model parameters, gradients, public keys, hyperparameters, etc, which, to some degree, alleviates concerns about data privacy.
To further satisfy different types of formal privacy protection requirements, various privacy protection techniques can be integrated into FL, such as Differential Privacy (DP)~\cite{triastcyn2019federated, wei2020nbafl}, Homomorphic Encryption (HE)~\cite{hardy2017private, fang2021large}, and Secure Multi-Party Computation (MPC)~\cite{melis2019exploiting, bonawitz2017practical}.
In short, the goal of FL is to jointly train a global model in a privacy-preserving manner and achieve a better performance compared to that without collaboration.

Formally, there are $M$ clients, and the $m$-th client owns a {\em private training dataset} $\mathcal{D}_m=\{(x_i^{(m)}, y_i^{(m)}) \in \mathcal{X} \times \mathcal{Y}, ~ i=1,2,\ldots,|\mathcal{D}_m|\}$, where $\mathcal{X}$ and $\mathcal{Y}$ are the input feature space and the label space, respectively. $\mathcal{D}_m$ is stored in the $m$-th client's private space, and $n=\sum_{m=1}^{M}|\mathcal{D}_m|$ is the total number of training instances.
Without sharing $\mathcal{D}_m$ directly with each other and the server, the $M$ clients together aim to train a model $h_{\theta}: \mathcal{X} \rightarrow \mathcal{Y}$ parameterized by $\theta$, with the {\em loss} $F: \mathcal{Y} \times \mathcal{Y} \rightarrow \mathbb{R}^{+} \cup \{0\}$. The {\em FL loss function} is:
\begin{equation}
    \mathcal{L} = \frac{1}{n}\sum_{m=1}^{M}\sum_{(x_i^{(m)}, y_i^{(m)}) \in \mathcal{D}_m} F\left(h_{\theta}(x_i^{(m)}),y_i^{(m)}\right). \label{eq:loss}
\end{equation}

\begin{figure*}[t]
    \centering
    \includegraphics[width=0.92\textwidth]{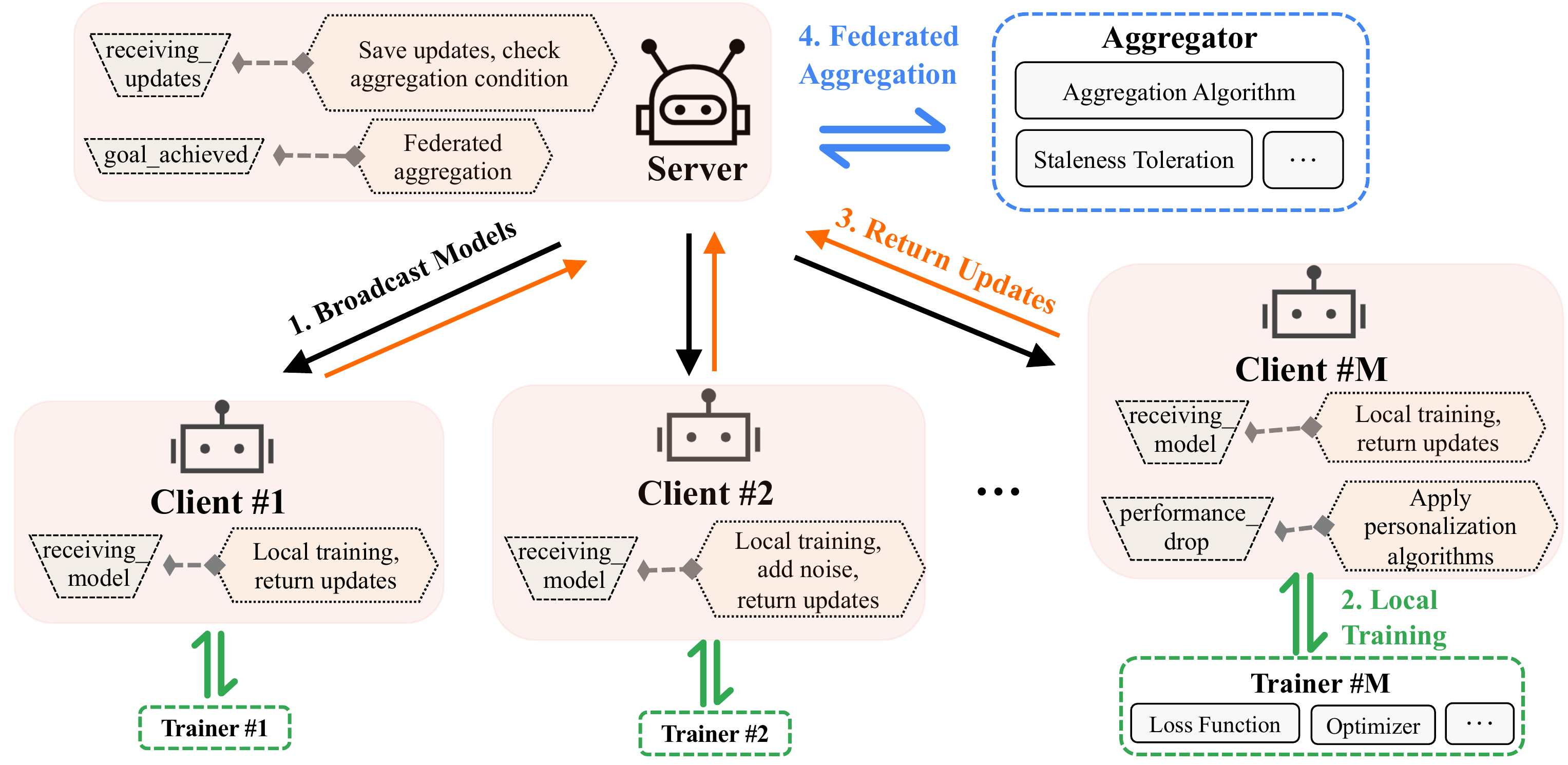}
    \vspace{-0.1in}
    \caption{Overview of an FL round implemented with FederatedScope.}
    \label{fig:oberview}
\end{figure*}

\stitle{Extensions.}
For the simplicity of presentation, we focus on a {\em vanilla FL} to minimize the loss function in Equation~\eqref{eq:loss} in most parts of this paper. Our \ours easily supports different federated settings in real-world FL applications, with more complicated loss functions, in order to handle the heterogeneity as discussed in Section~\ref{sec:intro}.
For example, for the purpose of personalization, the input feature spaces, the label spaces, and the underlying learning goals can be different for different clients. In \ours, clients can adopt different models and loss functions in local training, and only federally train the shared parts of the models. We will discuss more details in Section~\ref{subsec:p-m}.

\subsection{Related Works}
\stitle{Comparisons with existing FL platforms.}
In the recent years, growing along with the development of federated learning, federated learning platforms, including TFF~\cite{tff}, FATE~\cite{yang2019federated}, LEAF~\cite{leaf}, PySyft~\cite{pysyft}, FedML~\cite{fedml}, FedScale~\cite{lai2021fedscale}, etc., are proposed to support various kinds of applications.
These federated learning platforms provide data, models and algorithms, which saves users' effort on implementing from scratch and makes it easy for developers.
Most of these existing platforms adopt a procedural programming paradigm, which requires users to explicitly declare a sequential training process and computational graph from the global perspective. However, such a design makes the existing FL platforms unable to provide the required flexibility and extendability for the burgeoning demands from FL research and deployment, which limits the promotion and broadened the application scenarios of federated learning. Meanwhile, users also expect FL platforms to become more convenient and simple to handle the aforementioned heterogeneity in real-world FL applications.
To tackle these challenges, we propose \ours to provide great flexibility and extendability for users to handle the heterogeneity in FL. \ours provides rich implementations of FL algorithms for convenient usage, and provides different levels of programming interfaces for users to develop new algorithms.

\stitle{Comparisons with distributed machine learning.}
In distributed machine learning, the server has the rights to control the behaviors of clients; While in federated learning, all the participants could have their own behaviors following the achieved consensus to collaboratively train the model. For example, given the consensus that participants only need to share parts of the model parameters, clients can apply client-wise training configurations (e.g., training steps, learning rate, regularizer, optimizer, and so on) to locally train the model, and keep the non-shared model parameters and the learning goals (might be different among clients) private. To satisfy such requirements, \ours give the rights to the participants to describe behaviors from their own respective perspectives.
Besides, in federated learning, the quality, quantity, and distributions of clients' local data can be very diverse. Such heterogeneity in data makes it challenging to collaboratively learn, which motivates \ours to provide novel functionalities, such as asynchronous training strategies (in Section~\ref{subsec:asynchronous}) and personalized federated learning algorithms (in Section~\ref{subsec:p-m}), to make better use of their isolated data.
Furthermore, there exist more privacy/security protection requirements in federated learning compared to distributed machine learning. To tackle this, \ours provides some Byzantine fault tolerance algorithms to defend the malicious participants (in Section~\ref{sec:implementation}), and privacy protection component (in Section~\ref{subsec:privacy_protection}) and attack simulation component (in Section~\ref{subsec:privacy_attack}) to enhance and verify the privacy protection strength of FL applications.

%% file: subsections/3_event_driven.tex
\section{Design of FederatedScope}
\label{sec:methodoloy}
In this section, we introduce the design of \ours, showing how an FL training course can be framed and implemented using an event-driven architecture and why \ours makes it easy to handle heterogeneity in federated learning.
The overall structure is illustrated in Appendix~\ref{appendix:overvall_architecture}.

\subsection{Overview}
An FL course consists of multiple rounds of training, and a typical round implemented with \ours is illustrated in Figure~\ref{fig:oberview}, which includes four major steps:
(1) {\bf Broadcast Models:} the server broadcasts the up-to-date global model to the involved clients;
(2) {\bf Local Training:} once received the global model, clients perform local training using their trainer based on their private data;
(3) {\bf Return Updates:} after local training, clients return the model updates to the server;
(4) {\bf Federated Aggregation:} with the help of an aggregator, the server performs federated aggregation on the received model updates, and optimizes the global model.
To facilitate efficient development and deployment of such an FL course with multiple computation/communication rounds and different roles, there are two important design principles of \ours.
\squishlist
\item {\em Minimal dependency between different roles.} In \ours, each client or the server takes care of only the minimal portion of job it needs to collaboratively accomplish, including the model to be collaboratively learned and the exchanged messages. While allowing both synchronous and asynchronous training, we want to avoid introducing too much duty of coordinating and scheduling to the server. This is important especially when we consider the heterogeneity of resources and learning goals for the clients.
\item {\em Flexible and expressive programming interfaces for algorithm development and plug-in.} \ours aims to enable efficient development of FL algorithms via proper abstraction of FL courses and providing a necessary set of interfaces that developers need to implement.
Moreover, for the purpose of privacy protection and other functionalities, operators (e.g., for noise injection and encryption) and components (e.g., for auto-tuning) need to be plugged into the FL course in a flexible way. 
\squishend
Based on these principles, we first give an overview of our design.

\stitle{Basic infrastructure.} \ours employs an {\em event-driven} architecture within which the behaviors of different clients and the server in an FL course can be programmed (relatively) independently. The information exchange among participants and conditions to be checked by participants during the FL course are described as {\em events} (trapezoids within pink areas in Figure~\ref{fig:oberview}); when an event occurs, the corresponding {\em handlers} (hexagons within pink areas) that describes the behavior of a participant is triggered. For example, when ``{\sf receiving\_models}'' occurs, ``local training'' in a client is triggered; when ``{\sf goal\_achieved}'' occurs, ``federated aggregation'' in the server is triggered.
It turns out that the pairs of events and handlers are sufficiently expressive to describe all the existing (both synchronous and asynchronous) FL algorithms, as well as new ones we implement in \ours.

With such an infrastructure, \ours can easily support different machine learning backends (e.g., PyTorch~\cite{pytorch} and TensorFlow~\cite{tensorflow}). 
All users need to do is to transform exchanged information (called {\em messages}), which might be related to participants' local backends, into backend-independent ones before sharing, and parse the received messages according to the receiver's backend for further usage.
We call this procedure {\em message translation}.

\stitle{Programming interfaces.} Within the above infrastructure \ours provides, for each client or server, we only need to implement a {\em Trainer} (green dashed rectangles in Figure~\ref{fig:oberview}) or {\em Aggregator} (blue dashed rectangles), respectively, which encapsulates the details of local training or federated aggregation with well-defined interfaces, e.g., the loss function, optimizer, training step, aggregation algorithms, etc. A Trainer can be implemented as if a machine learning model is trained on the local data owned by a client.

Besides Trainer and Aggregator, the design of \ours allows flexible plug-in operators and components. For example, in order to ensure differential privacy, noise injection operators can be plugged to perturb the messages to be sent, where the amount of noise can be customized for different training tasks.
More details of the programming interfaces can be found in Section~\ref{sec:implementation}.

\stitle{Why event-driven?}
Benefited from such event-driven design, \ours provides users with expressiveness and flexibility to handle various types of heterogeneity in FL. Users would not be required to implement the FL courses from a centralized perspective as in a procedural programming paradigm, which might be rather complicated for real-world FL tasks due to the heterogeneity. Instead, each individual participant (a server or a client) is instantiated with its own events and handlers to independently describe its behaviors, such as what actions to take when receiving a certain type of message from others, how to perform local training and what information should be returned (for a client), and when to perform federated aggregation and start/terminate the training process (for a server).

In this way, the server performs federated aggregation under flexibly triggered conditions, which can prevent the training process from being blocked by unreliable or slow clients (more details in Section~\ref{subsec:asynchronous}). Different clients may customize their training configurations according to their own data distributions, tasks, and resources, such as training with different trainers for personalization (Section~\ref{subsub:personalized}), learning toward different goals (Section~\ref{subsub:multiple}), and running on different backends (Section~\ref{subsec:cross-backend}).
\ours also provides some native plug-in modules (Section~\ref{sec:plugins}) for various important functionalities, including privacy protection, attack simulation, and auto-tuning.
Before diving into these parts, we first provide more details about the event-driven design of \ours in Section~\ref{subsec:event_driven}.

\subsection{Event-driven Architecture}
\label{subsec:event_driven}
Event-driven architectures are widely adopted in distributed systems~\cite{eventdriven, kafka}.
With such an architecture, an FL training course in \ours can be framed into $<$event, handler$>$ pairs: the participants wait for certain {\em events} (e.g., receiving model parameters broadcast from the server) to trigger corresponding {\em handlers} (e.g., training models based on the local data). Hence, developers can express the behaviors of a participant (a server or a client) independently from its own perspective, rather than sequentially from a global perspective (considering all the participants together), and the implementations can be better modularized.

The events in \ours are categorized into two classes. One is related to message passing (e.g., exchanging information with others), which is also considered in previous FL platforms, e.g., receiving user-defined messages in FedML~\cite{fedml} and invoking requests in FedKeeper~\cite{chadha2020towards}.
The other class of events checks the satisfaction of customizable conditions (e.g., whether a pre-defined percentage of feedback from clients has been received). Some examples of events provided in \ours are presented in Appendix~\ref{appendix:events}.
Next we will introduce these types of events in more detail.

\ssstitle{Events Related to Message Passing.}
The exchanged information among participants are abstracted as messages, and an FL training course consists of several rounds of message passing.
Multiple types of messages are involved in an FL course, including but not limited to building up (e.g., \textit{join\_in} and \textit{id\_assignment}), training (e.g., \textit{model\_param} and \textit{gradients}), and evaluating (e.g., \textit{metrics}).
For the participants, receiving a message can be regarded as an event, and their follow-up behaviors can be described in handling functions (i.e., the handlers) to handle the received messages.
A handling function can be invoked by the event of receiving one or more types of messages, while receiving a certain type of message should only trigger one handling function directly. 

Take the vanilla FedAvg as an example, the clients' handling function for the event ``{\sf receiving\_models}'' can be described as ``\textit{train the received global model based on the local data, and then return the model updates}'', and the servers' handling function for the event ``{\sf receiving\_updates}'' can be described as ``\textit{save the model updates, and check whether all the feedback has been received}''. 

Generally, by defining the events related to message passing, \ours provides users with expressiveness to flexibly describe heterogeneous message exchange, such as exchanging model parameters, gradients, public keys, embeddings, generators, and so on. 
Meanwhile, through customizing the operations in the corresponding handlers, users can conveniently describe rich behaviors of participants, including training models based on the local data with personalized configurations, performing federated aggregation, predicting, clustering, generating, etc.
We will discuss more about this in Section~\ref{subsec:p-m}.

\ssstitle{Events Related to Condition Checking.}
Apart from the events related to message passing, the events related to condition checking are also indispensable for FL implementations. These events and the corresponding handlers describe the participants' behaviors when certain conditions are satisfied. For example, in an FL course, for the purpose of synchronization in training, the server checks whether the updated gradients or model parameters have been received from all the clients; if yes, it invokes an event ``{\sf all\_received}'', and this event triggers the federated aggregation and pushes forward the training process. 

One important usage of the events related to condition checking is to express the customizable conditions for triggering the federated aggregation. 
Besides ``{\sf all\_received}'', in order to support asynchronous training, \ours also provides events ``{\sf goal\_achieved}'' and ``{\sf time\_up}'' for such purpose.
Specifically, ``{\sf goal\_achieved}'' indicates that a certain percentage of feedback (so-called aggregation goal) has been received, and ``{\sf time\_up}'' denotes that the user-allocated time budget for each training round has run out.
Different from the event ``{\sf all\_received}'' that forces the server to wait for feedback from all the clients, ``{\sf goal\_achieved}'' allows the training process to move forward once the server has received enough feedback, while ``{\sf time\_up}'' encourages the server to collect as much feedback as possible within the time budget, both of which enable different asynchronous training strategies in FL.

Furthermore, the events related to condition checking also can be used to describe the behaviors of participants. For example, the server can be equipped with the events ``{\sf all\_joined\_in}'' (i.e., all the clients have joined in the FL course) and ``{\sf early\_stop}'' (i.e., pre-defined early stop conditions are satisfied) to describe when to start and terminate the training process, respectively, while the clients can use the events ``{\sf performance\_drop}'' to trigger personalization when the received global model causes the performance drop, and use ``{\sf low\_bandwidth}'' to reduce the communication frequency when the available bandwidth is not enough.

\ours provides warnings if there exist conflicts, and adopts a default resolution following the ``overwriting'' principle. Specifically, in an FL course implemented with \ours, each event is only permitted to be linked with one handler directly during the execution process. If an event is linked with more than one handler, which might cause conflicts in an FL course, a warning would be raised for users by \ours, and the latest linked handler would overwrite the older ones (e.g., the default handler is overwritten by the user-customized handlers). Finally, the handlers that take effect in an FL course would be printed out and recorded in the experimental logs. Users can remove some handlers or adjust the linked orders to make sure the intended handlers would take effect in the constructed FL courses.

\ours provides lots of predefined $<$event, handler$>$ pairs, which cover the rich implementation of existing FL algorithms, such as FedAvg~\cite{mcmahan2017communication}, personalization~\cite{pFedME, fedbn, li2021ditto}, federated graph learning approaches~\cite{wang2022federatedscopegnn}, and so on.
Users can implement their own algorithms based on these provided $<$event, handler$>$ pairs.
However it is out of our scope here to exhaustively list all the possible events related to message passing and condition checking.
The most important advantage is that the event-driven design of \ours provides users with expressiveness and flexibility to implement and customize diverse FL algorithms. 
Next, with \ours, we will demonstrate how to execute asynchronous federated training (Section \ref{subsec:asynchronous}), how to describe rich behaviors of the participants (Section \ref{subsec:p-m}) and how to conduct cross-backend FL (Section \ref{subsec:cross-backend}) in order to handle the heterogeneity of FL.

\begin{figure*}[t]
    \centering
    \includegraphics[width=0.90\textwidth]{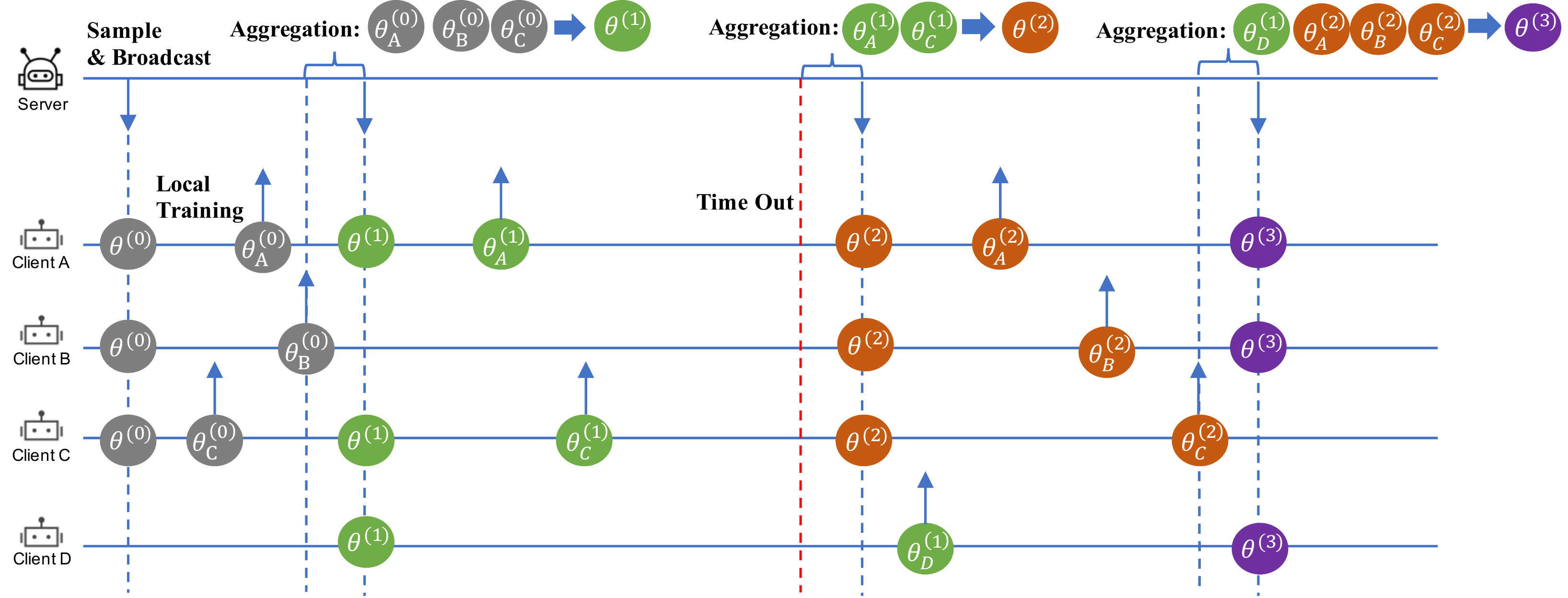}
    \vspace{-0.1in}
    \caption{An example of asynchronous training strategy in federated learning.}
    \label{fig:async}
\end{figure*}

\subsection{Supporting Asynchronous Training}
\label{subsec:asynchronous}
The asynchronous training strategies have been successfully applied in distributed machine learning to improve training efficiency~\cite{lian2015asynchronous, zhang2016staleness, chen2020asynchronous}.
Considering the aforementioned heterogeneity of FL in Section \ref{sec:intro}, the asynchronous training strategy is important to balance the model performance and training efficiency, especially in cross-device scenarios that involve a large number of unreliable and diverse clients. 
With the provided events and handlers, which specify what actions to take (i.e., handlers) when certain customizable conditions are satisfied (i.e., events), \ours supports users to conveniently design and implement suitable asynchronous training strategies for their FL applications.

\subsubsection{Behaviors for asynchronous FL}
\label{subsubsec:async_behavior}
Compared with the synchronous training strategy, several unique behaviors of participants might happen in asynchronous FL, which are modularized and provided in \ours as follows:

(i) \textit{Tolerating staleness in federated aggregation}. 
The term ``staleness'' denotes the version difference between the up-to-date global model maintained at the server and the model that a client starts from for local training, which should be tolerable to some extent in asynchronous FL. 
Specifically, in the federated aggregation, the staled updates from slow clients might be discounted in the aggregator but they still contribute to the aggregation. Of course, when the staleness is larger than a pre-defined threshold, the updates become outdated and thus can be directly dropped out.

(ii) \textit{Sampling clients with responsiveness-related strategies}.
The uniform strategy for sampling clients~\cite{mcmahan2017communication} might bring model bias in asynchronous FL, since the clients with low response speeds would contribute staled updates with higher probabilities compared with those who respond fast, which implies that the contributions of slow clients would be discounted or even dropped out in federated aggregation.
Similar phenomena are happened in synchronous FL using over-selection mechanism~\cite{tff}, as pointed out by previous studies~\cite{huba2022papaya, nguyen2022fedbuff, li2020federated}.
To tackle such an issue, with the prior knowledge of response speeds (it can be estimated from device information or historical responses), \ours provides a responsiveness-related sampling strategy (i.e., the sampled probabilities are related to the response speeds) and a group sampling strategy (i.e., clients with similar response speeds are grouped).

(iii) \textit{Broadcasting models after receiving update}. 
With the synchronous training strategy, the server broadcasts the up-to-date model to the sampled clients after performing federated aggregation.
Such a broadcasting manner, denoted as \textit{after aggregating} here, can also be adopted in asynchronous FL~\cite{wu2020safa}. 
We also provide another broadcasting manner to achieve asynchronous FL, named \textit{after receiving}~\cite{nguyen2022fedbuff}, in which the server sends out the current (up-to-date) model to a sampled idle client once the feedback is received. Compared with \textit{after aggregating}, the \textit{after receiving} manner can keep the consistent concurrency and promotes an efficient FL systems~\cite{huba2022papaya}. 

To the best of our knowledge, with the provided events and the corresponding handlers that describe the above behaviors, most of the existing studies on asynchronous FL can be conveniently implemented with \ours. 
For example, FedBuff~\cite{nguyen2022fedbuff} proposes to register the event ``{\sf goal\_achieved}'' and apply the \textit{after receiving} broadcasting manner, while SAFA~\cite{wu2020safa} suggests to equip \textit{after aggregating} broadcasting manner with event ``{\sf goal\_achieved}'' and manages clients based on their stalenesses.
Particularly, a synchronous FL course with the over-selection mechanism can be easily implemented in \ours by using event ``{\sf goal\_achieved}'' and setting the toleration to 0 (i.e., dropout all staled update).

In a nutshell, \ours is well-modularized toward flexibility and extensibility for handling the heterogeneity of FL via applying asynchronous training strategies.

\subsubsection{An example of asynchronous FL with \ours }

\begin{example}
An example of asynchronous training strategy for FL is demonstrated in Figure~\ref{fig:async}.
At the beginning of training process, the server samples a subset of the clients and broadcasts the model parameters (denoted as $\theta^{(i)}$ for the $i$-th round in the figure) to the sampled clients (e.g., Clients A/B/C in round $0$ and A/C/D in round $1$). These sampled clients perform local training based on their local data, and return the updated model parameters (e.g., $\theta^{(0)}_A$,  $\theta^{(0)}_B$ and $\theta^{(0)}_C$) to the server once they finish training. 

If the server could receive all the feedback from the sampled clients in time, as demonstrated in round $0$, event ``{\sf all\_received}'' happens, and the corresponding handler is triggered and federated aggregation is performed to generate the new global model (e.g., $\theta^{(1)}$) for the next training round.
 
However, in some cases, such as round $1$, some of the clients (e.g., Client D) fail to return the updated model in time due to some exceptions such as sluggish local training or device crash.
With the asynchronous training strategy implemented in \ours, when the allocated time budget has been run out, the ``{\sf time\_up}'' event would occur, and the server starts performing federated aggregation if the feedback has been received from a sufficient number of clients, otherwise the server executes ``remedial measures'', such as restarting the training round, sampling reliable clients additionally, or adaptively adjusting the time budget.
The staled feedback, such as $\theta^{(1)}_D$ in round $2$, can be saved and contributed to the aggregation if such staleness is still tolerable according to the user-defined threshold in the corresponding events. \hfill $\triangle$
\end{example}

As shown in this example, applying asynchronous training can improve the learning efficiency and ensure the effectiveness of the learned model when facing the aforementioned heterogeneity of FL.
More fancy strategies for asynchronous machine learning and FL~\cite{lian2015asynchronous, zhang2016staleness, chen2020asynchronous, wu2020safa} have been integrated into \ours, such as discounting the staled update, grouping clients according to their responsiveness, and so on.

\subsubsection{Convergence analysis}
We provide a theoretical analysis of convergence when applying the asynchronous training strategies in FL, under some widely-adopted assumptions~\cite{nguyen2022fedbuff, chai2021fedat, wang2021field} including smoothness and convexity of the loss function, the bounded gradient variances, and the bounded staleness.

\begin{proposition}
Consider the optimization problem defined in Equation~\eqref{eq:loss}. For each sampled client, it takes $Q$ local SGD steps with learning rate $\eta$ and returns model update to the server for federated aggregation.
Assume that the loss function $F$ is $L$-smooth and $\mu$-strongly convex, when setting $0< \mu Q \eta < 1$, the convergence of global model $\theta^{(t)}$ to the optimum $\theta^{(*)}$ after $T$ rounds satisfies:
\begin{align}
    \quad \mathbb{E}\big[F(\theta^{(T)}) - F(\theta^{(*)})\big] \leq (1-\mu Q \eta)^T\mathbb{E}\big[ F(\theta^{(0)}) - F(\theta^{(*)})\big]& \nonumber\\   
    +\frac{3LQ\eta}{\mu} \left(\sigma_l^2 + \sigma_g^2 + C\right)\left[\eta Q L (\tau_{\max}^2 +1)+\frac{1}{2}\right],&
\end{align}
where $\tau_{\max}$ denotes the maximum value of staleness caused by the asynchronous training strategies.
\end{proposition}

The above result shows the value of implementing asynchronous training strategies with \ours.
The detailed proof can be found in Appendix~\ref{appendix:proof}.
Compared to previous studies~\cite{chai2021fedat, wang2021field}, we extend the convergence analysis to more challenging asynchronous FL, where we quantify the effect of staled model updates for the federated aggregation and give a general analysis on the convergence rate rather than an ergodic version~\cite{nguyen2022fedbuff}.

\subsection{Supporting Personalization \& Multi-Goal}
\label{subsec:p-m}
In many real-world applications, handling the heterogeneity of FL brings the requirements of the flexibility of participants' training behaviors. That is, clients need client-specific training processes and/or different formats of loss functions to meet their resource limitations, data properties and learning goals, all of which can be diverse as discussed in Section \ref{sec:intro}.
Formally speaking, for the $m$-th client, the local training dataset $\mathcal{D}_m$ might correspond to client-specific feature space $\mathcal{X}_m$ and label space $\mathcal{Y}_m$,
which can lead to sub-optimal performance of the global model $h_\theta$ or even makes it unusable.
To tackle this, the client could (1) maintain a local model $h_{\theta_m}$ with personalized parameters $\theta_m$ (i.e., personalization) and/or (2) minimize the local loss function $F_m$ (i.e., multiple learning goals), while only sharing parts of the models with others for federal training. Therefore, the loss function in Equation~\eqref{eq:loss} can be extended as:
\begin{equation}
    \mathcal{L}' = \frac{1}{n}\sum_{m=1}^{M}\sum_{(x_i^{(m)}, y_i^{(m)}) \in \mathcal{D}_m} F_m\left(h_{\theta_m}(x_i^{(m)}),y_i^{(m)}\right). \label{eq:pfl_loss}
\end{equation}
Note that there exists some shared parameters among clients, i.e., $\bigcap\limits_{m=1}^{M}\theta_m \neq \emptyset $, and all the clients collaboratively learn $\theta_1, \theta_2, \ldots, \theta_M$ to jointly minimize $\mathcal{L}'$.

Benefited from the event-driven architecture, \ours provides users with flexible expressiveness to describe the behavior of an individual participant (a server or a client) from its own perspective, which is crucial for handling the heterogeneity of FL via allowing the differences among participants. 
In this section, we present how \ours supports such differences among participants for handling the heterogeneity of FL through the following two ways.

\subsubsection{Personalized training behaviors}
\label{subsub:personalized}
As discussed by previous work~\cite{fallah2020personalized,tan2021towards,chen2022pfl}, the heterogeneity of FL (e.g., the heterogeneity in local data and participants' resources) might hurt the model performance for some clients and lead to the sub-optimal performance when sharing the same global model among all participants, such as vanilla FedAvg~\cite{mcmahan2017communication}, which motivates the study of personalized federated algorithms~\cite{pFedME, fedbn, li2021ditto, fedem}.
Specifically, personalized federated algorithms are proposed to apply client-specific local training courses based on their private data, including client-wise training configuration, sub-modules, global-local fusing weights, etc. 
The convergence of the federal learning process is not be determined solely by the learning process of the global model. Each participant can independently choose the most suitable snapshot of the global model.
Therefore, users are expected to describe diverse behaviors of clients to develop personalized federated algorithms, which might be rather complicated and inconvenient when using a procedural programming paradigm since lots of effort is put into sequentially coordinating and describing the participants.

With the event-driven architecture, \ours allows users to describe the behaviors of participants independently, which provides great flexibility to develop new personalization algorithms.
Users are able but not limited to (1) specify the training configurations, such as local training steps and learning rate, for an individual client; (2) define new events related to new types of exchanged messages and/or events related to customized conditions to apply personalization algorithms (e.g., {\sf performance\_drop}); (3) add personalized behaviors into handlers that are triggered for local training, such as fusing the received global model with local models before performing local training.
In most cases, such customization can be inherited from the general training behaviors and only need to focus on the differences.
Considering that clients might have different privacy protection requirements, some privacy protection techniques can be adopted. For example, clients might choose to inject noise into the model parameters before sharing them. More details of the privacy protection of messages can be found in Section~\ref{subsec:privacy_protection}.

We provide several representative personalized federated algorithms~\cite{chen2022pfl} in \ours for handling the heterogeneity in FL, including pFedMe~\cite{pFedME}, FedBN~\cite{fedbn}, FedEM~\cite{fedem}, and Ditto~\cite{li2021ditto}.
These built-in algorithms serve as examples for showing how to easily and flexibly develop new personalized federated algorithms, and can also be conveniently adopted via configuring by users in real-world applications.

\subsubsection{Multiple Learning Goals}
\label{subsub:multiple}
Note that the scope of FL also covers the scenarios where participants learn common knowledge while optimizing different learning goals~\cite{xie2021federated, fedem, smith2017federated, yao2022federated}. The participants of an FL course reach a consensus on what needs to be shared while keeping other learning parts private, especially in cross-silo scenarios. For example, several medical research institutes would like to collaboratively learn a graph neural network for capturing the common structure knowledge of molecules, but they will not disclose what is the usage of the learned structure knowledge.
They might exchange the update of the graph convolution layers while maintaining the encode layers, readout layers, and headers (such as classifier) private. In this and more similar scenarios from model pre-training, it can be difficult or even intractable for users to develop with a procedural programming paradigm via defining the static computation graph of the FL course. 

Fortunately, the event-driven design of \ours makes it easy to express and implement the FL courses with multiple learning goals. Each participant owns its local model and private data, defines its computation graph, locally trains with private learning objective, and only exchanges messages of the shared layers with others through FL.

Currently, \ours provides three representative scenarios of FL with multiple learning goals, including graph classification, molecular property inference, and natural language understanding (NLU). In the graph classification scenario, clients own different graph classification tasks and aim to collaboratively improve their own performance due to the limitation of available training data. In the molecular property inference scenario, different clients have different property inference goals, such as the solubility (regression task), the enzyme type (classification task), and the penetration (classification task), which leads to heterogeneity in terms of task type. In the NLU scenario, clients are also heterogeneous in terms of task type, and they own different NLU tasks, including sentiment classification, reading compression, and sentence pair similarity prediction. Since the development of FL with multiple learning goals is still in the early stage, \ours provides these scenarios to broaden the scope of FL applications and promote the development of innovative methods.
More details of these scenarios of FL with multiple learning goals can be found in our open source repository~\cite{ExampleMultiLearningGoal}.

In summary, \ours allows users to describe participants' behaviors from their respective perspectives and thus provides flexibility in applying different training processes and learning goals to the participants to handle the heterogeneity of FL.

\subsection{Supporting Cross-backend FL}
\label{subsec:cross-backend}
Motivated by the strong need from real-world applications, \ours supports constructing cross-backend FL courses. For example, in an FL task, some of the involved clients are equipped with {TensorFlow} while others might run with {PyTorch}. 
Thanks to the event-driven architecture, \ours can conveniently provide such functionality via a mechanism called \textit{message translation}. 
Note that such support of cross-backend FL is different from those provided by the universal languages such as ONNX~\cite{onnx} and the existing FL platforms such as TFF~\cite{tff}.

Conceptually, ONNX and TFF adopt a global perspective of constructing an FL course, which implies that the complete computation graph is globally defined and shared among all participants. In order to make it compatible with different (versions of) machine learning backends on different clients, the global computation graph is serialized into platform-independent and language-independent representations, sent to the clients, and interpreted or compiled accordingly for different backends.

\stitle{Message translation.}
\ours, in contrast, gives each participant the right of describing the computation graph on its own (for the portion it takes charge of).
Hence, participants can define the computation graph based on their running backends.
Following a pre-defined consensus on the format of messages, the participants transform the messages, e.g., gradients and model parameters, generated from the local backends into the pre-defined backend-independent format, e.g., an array of pairs of parameters and values, before sharing them with others.
This procedure is called {\em encoding}.
For the other direction, once an encoded message is received, the participant parses the message, e.g., the above array, into backend-dependent tensors in its own computation graph and backend, which is called the {\em decoding} procedure.

The encoding and decoding procedures are abstracted as two special programming interfaces in \ours with default implementations; they can also be customized for each participant based on its backend and the FL algorithm to be deployed.
\ours provides several examples of constructing cross-backend federated learning~\cite{ExampleCrossBackend}.

In supporting cross-backend FL, the advantage of \ours is two-fold: (1) \ours provides more flexibility to handle the heterogeneity of FL than other platforms that adopt a global perspective since each participant has the right to declare its computation graph independently. Specifically, the developer of each participant can focus on expressing its own computation graph, such as client-specific embedding layers and output layers, adapting to its input instance and task. There is no need to declare a super graph (i.e., the global perspective) and care about how to distribute it, reducing the implementation difficulty. (2) \ours follows the principle of information minimization, where participants only need to achieve a consensus on the format of messages and exchange necessary information. Thus, the exchanged model parameters will not leak the whole model architecture, the local training algorithm, or the personalization-related operators to other participants, which would otherwise be inferable from the global computation graph of ONNX and TFF. When such information leakage happens, malicious participants benefit from it because they can conduct a white-box attack rather than the more challenging black-box one in \ours. We will talk more about privacy attacks in Section~\ref{subsec:privacy_protection}.

\subsection{Usage of FederatedScope}
\label{sec:implementation}

In this section, we give a full example of how to set up an FL course, so that users can gain a clear and vivid understanding of \ours. At a high level, users should define a series of events and their corresponding handlers, which characterize the behaviors of participants. As shown in Figure~\ref{code:event_handler}, the handlers are expressed as callable functions and bound to the corresponding events with a register mechanism. When an event happens, the corresponding handler will be called to handle it. The example is as follows:

\begin{figure}[t]
\centering
\noindent\rule{0.474\textwidth}{0.8pt}
\begin{minted}[tabsize=2,breaklines, fontsize=\footnotesize]{python}
class CustomizedServer(Server):
    def customized_handler(args):
        Do sth. # Describe the operations for handling the event
    ... ...
    # Register the customized handlers for customized events
    registered_handlers = dict() # Expected type {event: handler}
    register(customized_event, customized_handler)
    ... ...
    if customized_event == True:
        # Call the corresponding customized_handler
        registered_handlers[customized_event](args)
\end{minted}
\vspace{-0.1in}
\noindent\rule{0.474\textwidth}{0.8pt}
\vspace{-0.2in}
\caption{Behaviors description with events and handlers.}
\label{code:event_handler}
\end{figure}

\begin{example}
Consider that a server and several clients would like to construct an FL course and they agree to exchange certain model parameters during the training process.

For clients, the event related to message passing is ``{\sf receiving\_models}'', and the corresponding handler can be ``\textit{train the received global models based on local data, and then return the model updates}''.
The local training process is executed by a \textit{Trainer} object held by the client. As illustrated in Figure~\ref{code:trainer}, the trainer encapsulates the training details, entirely decoupled from the client's behaviors. Hence, the training process can be described as those of the centralized learning case, and the trainer can be flexibly extended with fancy optimizers, regularizers, personalized algorithms, etc. Such a design makes it easy for user customizations.

\begin{figure}[t]
\centering
\noindent\rule{0.474\textwidth}{0.8pt}
\begin{minted}[tabsize=2,breaklines, fontsize=\footnotesize]{python}
class Client(object):
    trainer = CustomizedTrainer(args)
    ... ...
    def handler_for_receiving_models(args):
        # Perform local training when receiving the global models
        model_update = trainer.train(args.model, args.data)
        send(message=model_update, receiver=server)
\end{minted}
\vspace{-0.1in}
\noindent\rule{0.474\textwidth}{0.4pt}\\[\dimexpr-\baselineskip+1mm+0.5pt]
\noindent\rule{0.474\textwidth}{0.4pt}
\begin{minted}[tabsize=2,breaklines, fontsize=\footnotesize]{python}
class CustomizedTrainer(Trainer):
    ... ...
    # Describe training behaviors (same as centralized training)
    def train(received_models, data):
        # Personalized algorithms might be applied here
        local_model = update_from_global_models(received_models)
        preds = local_model.forward(data.x)
        args = [optimizer, loss_function, regularizer, ...]
        model_updates = local_model.backward(data.y, preds, args)
        return model_updates
\end{minted}
\vspace{-0.1in}
\noindent\rule{0.474\textwidth}{0.8pt}
\vspace{-0.2in}
\caption{The training behaviors and clients are decoupled for supporting flexible customization.}
\label{code:trainer}
\end{figure}

For the server, the event related to message passing is ``{\sf receiving\_updates}'' and the corresponding handler can be ``\textit{save the model updates, and check the aggregation condition}'', which requires another event related to condition checking. For the synchronous training strategy, such event can be ``{\sf all\_received}'' and the corresponding handler will be ``\textit{perform federated aggregation, and broadcast the updated global models}''.
For the asynchronous training strategies, the event ``{\sf all\_received}'' can be replaced with ``{\sf goal\_achieved}'' or ``{\sf time\_up}'', which adds flexible behaviors during sampling clients or performing aggregation (More details can be found in Section~\ref{subsec:asynchronous}).
The federated aggregation is executed by an aggregator, which is also decoupled with the server for flexibly supporting various state-of-the-art (SOTA) aggregation algorithms, such as FedOpt~\cite{FedOPT2020Asad}, FedNova~\cite{FedNova2020Wang}, FedProx~\cite{blocal}, etc.

Note that when events such as  ``{\sf all\_received}'' or ``{\sf goal\_achieved}'' happens, the clients would receive the up-to-date global models after the server performs federated aggregation, which naturally causes the following event ``{\sf receiving\_models}'' and triggers the handlers for performing a new round of local training.
In this way, although we have not explicitly declared a sequential training process, the events happen in the intended logical order to trigger the corresponding handlers, which can precisely express the FL procedure and promote modularization.
Further, events such as ``{\sf maximum\_iterations\_reached}'' or ``{\sf early\_stopped}'' can be adopted to specify when the FL courses should be terminated. \hfill $\triangle$
\end{example}

With such event-driven architecture, \ours allows users to use existing or add new $<$event, handler$>$ pairs for flexible customization, rather than inserting the new behaviors into the sequential FL course carefully as those in the procedural programming paradigm.
For example, by simply changing the event ``{\sf all\_received}'' to other events related to condition checking such as ``{\sf goal\_achieved}'', users can conveniently apply asynchronous training strategies. Users also can add some new events related to message passing to enable the heterogeneous information exchange, such as node embeddings in graph federated learning~\cite{xie2021federated} and encrypted results in cross-silo federated learning~\cite{hardy2017private}.
Several representative examples from end-user perspective can be found on our website, including constructing FL courses via simple configuring~\cite{ExampleConfiguring} and developing new functionalities~\cite{ExampleStartOwnCase}.

\stitle{Trainer \& Aggregator.}
The details of the adopted algorithms in trainer and aggregator are decoupled with the behaviors of participants. Therefore, when users develop their own trainer/aggregator with \ours, they only need to care about the details of training/aggregating algorithms. For example, users are expected to implement several basic interfaces of trainers, including train, evaluation, update model, etc., which is the same as those in centralized training and serves as ``must-do'' items. For the aggregator, which takes the received messages as inputs and returns the aggregated results, users only need to implement how to aggregate.

\stitle{Programming Interfaces and Completeness Checking.}
\ours provides base classes to aware users of the necessary interfaces for an FL course, such as {\sf BaseTrainer} and in {\sf BaseWorker}. These base classes can be used to check the completeness of the defined FL courses, since an ``Not Implementation Error'' would be raised to abort the execution if users fail to implement the necessary interfaces. 
With the base classes, \ours provides rich implementation of existing FL algorithms. Therefore users can inherit the provided implementation and focus on the development of new functions and algorithms, which also ensures the completeness of FL courses.
Besides, \ours provides a completeness checking mechanism to generate a directed graph to verify the flow of message transmission in the constructed FL course (an example is illustrated in Appendix~\ref{appendix:completeness}).

\stitle{Robustness Against Malicious Participants.}
To defend malicious participants and make the system more robust, some Byzantine fault tolerance algorithms are provided in \ours. 
For example, we can apply the Krum~\cite{blanchard2017machine} aggregation rule in federated aggregation.
Note that these Byzantine fault tolerance algorithms can be regarded as the aggregation behaviors of server and implemented in the aggregator, which is decoupled with other behaviors to make it flexible and extendable for users to develop their own fault tolerance algorithms.

%% file: subsections/4_plugins.tex
\section{Important Plug-In Components}
\label{sec:plugins}
In this section, we present several important plug-in components in \ours for convenient usage.
These components provide functionalities including privacy protection, attack simulation, and auto-tuning, all of which are tightly coupled with the design of \ours and serve as plug-ins.

\subsection{Behavior Plug-In: Privacy Protection}
\label{subsec:privacy_protection}
Real-world FL applications might prefer different privacy protection algorithms due to their diversity in types of private information, protection strengths, computation and communication resources, etc., which motivates us to provide various privacy protection algorithms in \ours.

With the design of \ours, privacy protection algorithms can be implemented as behavior plug-ins, which indicates that the privacy protection algorithms bring new behaviors of participants.
For example, before the participants share messages, the encryption algorithms might be applied on the messages, or the messages would be partitioned into several frames, or certain noise can be injected into the messages. 
These behaviors have been pre-defined in \ours (so-called the behavior plug-in), and can be easily called to protect privacy via simple configuration. 

\begin{figure}[t]
\centering
\noindent\rule{0.474\textwidth}{0.8pt}
\begin{minted}[tabsize=2,breaklines, fontsize=\footnotesize]{python}
class Client(object):
    def handler_for_receiving_models(args):
        ... ...
        if config.inject_noise_before_sharing == True:
            # Inject certain noise before sharing the message
            args = [noise_distribution, budget, ...]
            protected_messages = add_noise(messages, args)
            send(message=protected_messages, receiver=server)
        else:
            send(message=messages, receiver=server)
\end{minted}
\vspace{-0.1in}
\noindent\rule{0.474\textwidth}{0.8pt}
\vspace{-0.2in}
\caption{Behavior plug-in: injecting noise.}
\label{code:inject_noise}
\end{figure}

Specifically, we implement a widely-used homomorphic encryption algorithm Paillier~\cite{paillier1999public} and apply it in a cross-silo FL task~\cite{hardy2017private}; and we develop a secret sharing mechanism for FedAvg.
These provided examples demonstrate how to apply privacy protection algorithms with \ours. 
Furthermore, to satisfy the heterogeneity in privacy protection strengths, we provide tunable modules for applying Differential Privacy (DP) in FL, which has been a popular technique for privacy protection and has achieved great success in database and FL applications \cite{dp_survey, ding2011differentially, triastcyn2019federated, wei2020nbafl}.
An example is illustrated in Figure~\ref{code:inject_noise}, from which we can see that users can utilize the configuration to modify the client's behavior: injecting certain noise into the messages before sharing.
Users can combine different behaviors together to implement fancy DP algorithms such as NbAFL~\cite{wei2020nbafl}. Note that to achieve a theoretical guarantee of privacy protection, users still need to specify some necessary settings according to their own data and tasks, including the noise distribution~\cite{phan2017adaptive, dwork2014algorithmic} and privacy budget allocation~\cite{wang2019answering, li2021federated, luo2021privacy}.

\subsection{Participant Plug-In: Attack Simulation}
\label{subsec:privacy_attack}
Attacks, growing along with the development of FL, are important for users to verify the availability and the privacy protection strength of their FL systems and algorithms. Typical attacks include privacy attack and performance attack: the former aims to steal the information related to clients' private data, while the latter aims to intentionally guide the learned model to misclassify a specific subset of data for malicious purposes such as back-door.
However, most of the existing FL platforms ignore such an important functional component.

Note that it is non-trivial to provide attack simulation in an FL platform, since the diversity of privacy and performance attacks brings challenges to the platform's flexibility and extensibility.
Benefited from the design of \ours, the behaviors of malicious participants can be expressed independently, thus the attack simulation can be implemented as the participant plug-in in \ours.
To be more specific, as shown in Figure~\ref{code:attck}, users can conveniently choose some of the participants to become malicious clients via configuring, and attack algorithms can be added to their own trainers.
These malicious clients are able to collect or inject certain messages among victims, and further recover or infer the target information accordingly.
The simulated attacks provided in \ours can be used to verify the privacy protection strength of their FL systems and algorithms. For example, when users develop a new FL algorithm, they want to know the protection level of the proposed algorithm from some perspectives, such as whether the dataset properties or private training samples would be inferred by attacks. They can use several state-of-the-art attack algorithms, which have been provided in \ours for convenient usage, to check the privacy protection strength of their FL algorithms, and enhance the privacy protection strength if necessary according to the results of simulated attacks.

\ours provides rich types of attack. 
For privacy attack, \ours provides the implementation of the following algorithms:
(i) Gradient inversion attack~\cite{nasr2019comprehensiveIG} for membership inference; 
(ii) PIA~\cite{melis2019exploitingPIA} for property inference attack; 
(iii) DMU-GAN~\cite{hitaj2017deep} for class representative attack; 
(iv) DLG~\cite{zhu2019dlg}, iDLG~\cite{zhao2020idlg}, GRADINV~\cite{geiping2020inverting} for training data/label inference attack. 
In terms of performance attack, \ours currently focuses on the backdoor attack, a representative type of performance attack, whose objective is to mislead the model to classify some selected samples to the attacker-specified class. The implementations of SOTA backdoor attacks include: (i) Edge-case backdoor attacks~\cite{edge_case_bd}, BadNets~\cite{badnet}, Blended~\cite{blended}, WaNet~\cite{wanet}, NARCISSUS~\cite{narc}, which perform back-door attack by poisoning the dataset; (ii) Neurotoxin~\cite{neurotoxin} and DBA~\cite{dba}, which perform back-door attack by poisoning the model.

\begin{figure}[t]
\centering
\noindent\rule{0.474\textwidth}{0.8pt}
\begin{minted}[tabsize=2,breaklines, fontsize=\footnotesize]{python}
class Fed_Runner(object):
    ... ...
    def setup_client(config):
        if config.is_malicious == True:
            # Instantiate a malicious client with attack behavior
            client = MaliciousClient(attack_algorithms, args)
        else:
            # Instantiate a normal client
            client = Client(args)
        client.join_in_FL_course()
\end{minted}
\vspace{-0.1in}
\noindent\rule{0.474\textwidth}{0.8pt}
\vspace{-0.2in}
\caption{Participant plug-in: malicious client.}
\label{code:attck}
\end{figure}

\subsection{Manager Plug-In: Auto-tuning}
\label{subsec:auto}
FL algorithms generally expose hyperparameters that can significantly affect their performance. Without suitable configurations, users cannot manage their FL applications well. Hyperparameter optimization (HPO) methods, both traditional methods (e.g., Bayesian optimization~\cite{BO} and multi-fidelity methods~\cite{li2017hyperband, bohb, dehb, optuna}) and Federated-HPO methods~\cite{fedbo, fedex} (denoting very recent ones that deliberately take the FL setting into account) can help users manage FL applications by automatically seeking suitable hyperparameter configurations.

Therefore, in \ours, we provide an auto-tuning plug-in, which incorporates various HPO methods. Conceptually, Bayesian optimization, multi-fidelity, and Federated-HPO methods treat a complete FL course, a few FL rounds, and client-wise local update procedures as black-box functions to be evaluated, respectively. \ours provides a unified interface to manage the underlying FL procedure in various granularities so that different HPO methods can interplay with their corresponding black-box functions. This unification is nontrivial for the last case, where we leverage our event-driven architecture to achieve the client-wise exploration of Federated-HPO methods. When they are plugged in, the exchanged messages are extended with HPO-related samples/models/feedback, and the participants would handle them with extended behaviors accordingly.

\begin{figure}[t]
\centering
\noindent\rule{0.474\textwidth}{0.8pt}
\begin{minted}[tabsize=2,breaklines, fontsize=\footnotesize]{python}
class Server(object):
    def handler_for_receiving_updates(args):
        ... ...
        if config.apply_fedex == True:
            # Choose hyperparameters for the client
            cfg = sample_cfg(cfg_candidates, args.hpo_feedback)
        # Continue to handling the message accordingly
        ... ...
\end{minted}
\vspace{-0.1in}
\noindent\rule{0.474\textwidth}{0.4pt}\\[\dimexpr-\baselineskip+1mm+0.5pt]
\noindent\rule{0.474\textwidth}{0.4pt}
\begin{minted}[tabsize=2,breaklines, fontsize=\footnotesize]{python}
class Client(object):
    def handler_for_receiving_models(args):
        ... ...
        if config.apply_fedex == True:
            # Apply the received hyperparameters
            trainer.apply_cfg(args.received_config)
        # Continue to handling the message accordingly
        ... ...
\end{minted}
\vspace{-0.1in}
\noindent\rule{0.474\textwidth}{0.8pt}
\vspace{-0.2in}
\caption{Manager plug-in: re-specify configuration.}
\label{code:hpo}
\end{figure}

For Bayesian optimization methods, we showcase applying various open-sourced HPO packages to interact with \ours. Each time they propose a specific configuration, \ours executes an FL course accordingly and returns a specified metric (e.g., validation loss) as the function's output.

As for multi-fidelity methods, we have implemented Hyperband~\cite{li2017hyperband} and PBT~\cite{pbt} in \ours. Specifically, \ours can export the snapshot of a training course to a corresponding checkpoint, from which another training course can restore. With such a checkpoint mechanism, these multi-fidelity methods can evaluate the configurations that have survived previous low-fidelity comparisons by restoring from the last checkpoints rather than learning from scratch.

\begin{table*}[t]
    \centering
    \caption{The comparison between applying synchronous and asynchronous training strategies in federated learning, in terms of the virtual time cost (hours) to achieve the targeted test accuracy.  }
    \vspace{-0.05in}
    \begin{tabular}{l c c c c c c c}
    \toprule
    \multirow{2}{*}{Dataset (Target Acc.)} & \multicolumn{3}{c}{Sync.} & \multicolumn{4}{c}{Async.} \\ 
    \cmidrule(lr){2-4} \cmidrule(lr){5-8} 
     & \multicolumn{1}{c}{Vanilla} & \multicolumn{1}{c}{OS} & \multicolumn{1}{c}{OS (FedScale)} & \multicolumn{1}{c}{Goal-Aggr-Unif} & \multicolumn{1}{c}{Goal-Rece-Unif} & \multicolumn{1}{c}{Time-Aggr-Unif} & \multicolumn{1}{c}{Goal-Aggr-Group}\\
    \midrule
    FEMNIST (85\%) & $61.46$ & $27.34_{\;2.25\times}$ & $28.78_{\;2.14\times}$ & $11.29_{\;5.44\times}$ & $11.36_{\;5.41\times}$ & $11.70_{\;5.25\times}$ & $10.42_{\;5.90\times}$\\
    CIFAR-10 (70\%) & $66.99$ & $26.42_{\;2.54\times}$ & $28.98_{\;2.31\times}$& $7.73_{\;8.67\times}$ & $7.98_{\;8.39\times}$ & $8.87_{\;7.55\times}$ & $7.54_{\;8.88\times}$ \\
    Twitter (69\%) & $9.41$ & $3.84_{\;2.45\times}$ & $4.14_{\;2.27\times}$ & $0.78_{\;12.06\times}$ & $0.64_{\;14.70\times}$ & $0.50_{\;18.82\times}$ & $0.65_{\;14.48\times}$ \\
    \bottomrule
    \end{tabular}
    \label{table:main_comparison}
\end{table*}

Furthermore, \ours provides FedEx~\cite{fedex} as an exemplary implementation of Federated-HPO methods. Specifically, once FedEx is plugged in, we sample configurations for each client independently in each FL round. Then each client re-specifies its native configuration and conducts local updates accordingly, as shown in Figure~\ref{code:hpo}. Finally, the client-wise feedback is aggregated to update the policies responsible for determining the optimal configuration(s).

In summary, the auto-tuning plug-in can manage FL applications in various granularities. Traditional HPO methods interplay with \ours by configuring and running one or more complete FL rounds, while Federated-HPO methods explore client-wise configurations concurrently in a single FL round. With flexibility provided by the event-driven architecture, we have implemented these HPO methods in a unified way~\cite{wang2022fedhpo}, and novel HPO methods can be easily developed and contributed to \ours.

%% file: subsections/5_exp.tex
\section{Experiments}
\label{sec:exp}

\subsection{DataZoo and ModelZoo}
For convenient usage, we collect and preprocess ten widely-used datasets from various FL application scenarios, including computer vision datasets (FEMNIST~\cite{emnist}, CelebA~\cite{celeba} and CIFAR-10~\cite{cifar10}), natural language processing datasets (Shakespeare~\cite{mcmahan2017communication}, Twitter~\cite{twitter} and Reddit~\cite{Reddit}) from LEAF~\cite{leaf}, and graph learning datasets (DBLP~\cite{dblp_dataset}, Ciao~\cite{ciao_dataset} and MultiTask~\cite{xie2021federated}) from FederatedScope-GNN (FS-G)~\cite{wang2022federatedscopegnn}. 
The statistics of these datasets can be found in Appendix~\ref{appendix:datazoo}.
Meanwhile, we provide off-the-shelf neural network models via our ModelZoo, which includes widely-adopted model architectures, such as ConvNet~\cite{convnet} and VGG~\cite{vgg} for computer vision tasks, BERT~\cite{bert} and LSTM~\cite{lstm} for natural language processing tasks, and various GNNs~\cite{gcn, gat, gin, sage, gpr-gnn} for graph learning.
Such ModelZoo allows users to conveniently develop various trainers for clients.

\begin{figure*}
    \begin{minipage}[t]{0.33\textwidth}
    \setcaptionwidth{0.96\textwidth}
    \centering
    \includegraphics[width=0.96\textwidth]{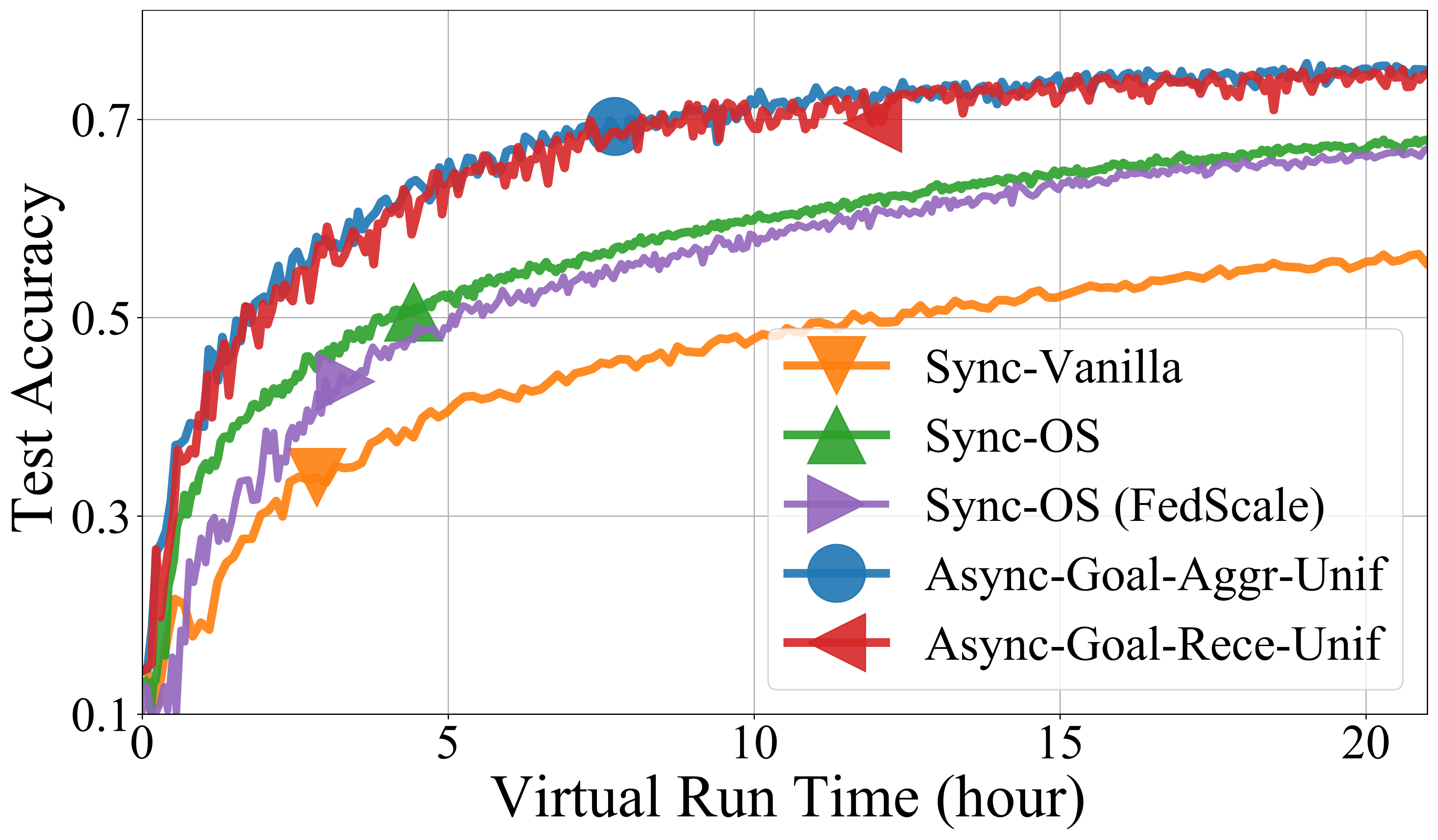}
    \caption{The comparison between synchronous and asynchronous strategies.}
    \label{fig:async_results}
    \end{minipage}
     \begin{minipage}[t]{0.33\textwidth}
     \setcaptionwidth{0.96\textwidth}
    \centering
    \includegraphics[width=0.96\textwidth]{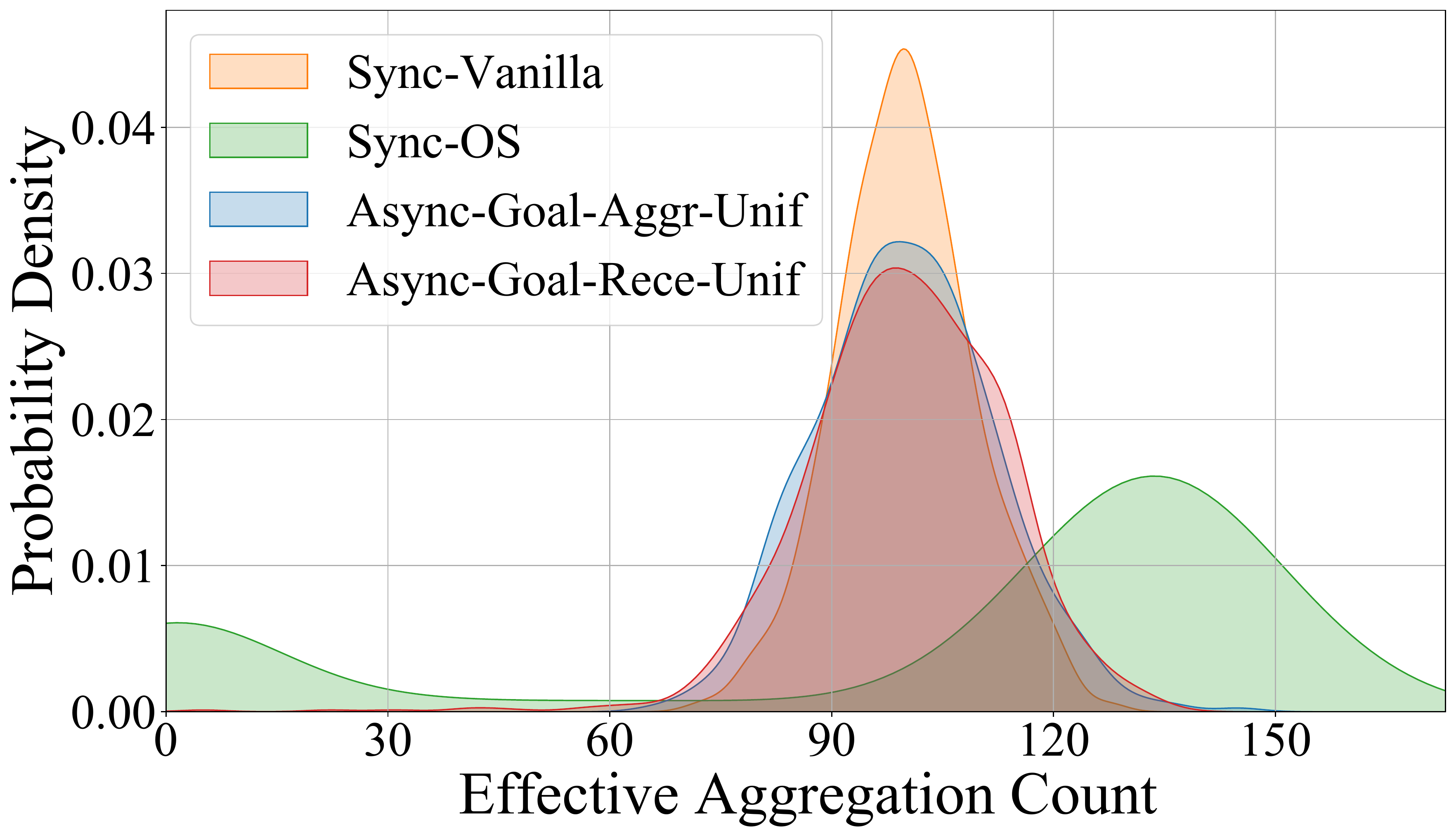}
    \caption{The distributions of the aggregated count of the clients.}
    \label{fig:agg_count}
    \end{minipage}
     \begin{minipage}[t]{0.33\textwidth}
     \setcaptionwidth{0.96\textwidth}
    \centering
    \includegraphics[width=0.96\textwidth]{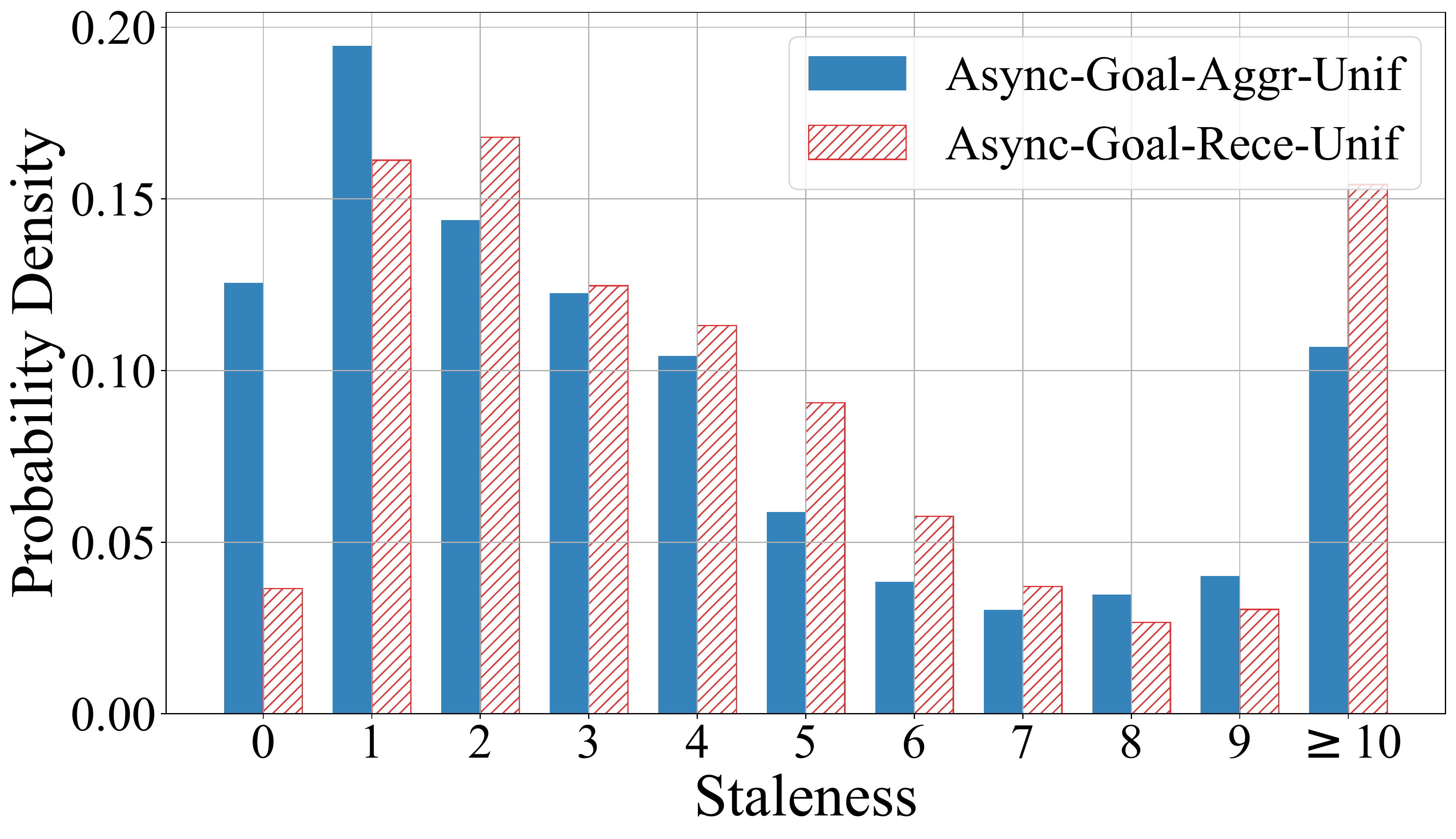}
    \caption{The distributions of the staleness in asynchronous strategies.}
    \label{fig:staleness}
    \end{minipage}
\end{figure*}

\subsection{Experiment Settings}
Here we conduct a series of experiments with \ours on three representative datasets as follows:

\noindent \textbf{FEMNIST}. FEMNIST consists of 805,263 handwritten digits in 62 classes, which are partitioned into 3,597 clients according to the writers. With FL, a CNN with two convolutional layers is trained for image classification task on this dataset.

\noindent \textbf{CIFAR-10}. 
As suggested by previous studies~\cite{lda}, we partition the dataset into 1,000 clients with a Dirichlet distribution, and federally train a CNN with two convolutional layers for image classification.

\noindent \textbf{Twitter}. We sample a subset from Twitter, which consists of 6,602 twitter users' 16,077 texts. Each twitter user can be regarded as a client for constructing an FL course. Following previous study~\cite{leaf}, we embed the texts with a bag-of-words model~\cite{bagofword} and collaboratively train a logistic regression model on these sampled clients for sentiment analysis.

More implementation details can be found in Appendix~\ref{appendix:implementation} and ~\ref{appendix:programming_amount}.

\subsection{Results and Analysis}
\subsubsection{Asynchronous Federated Learning}
\label{exp:async}
We first conduct experiments to compare the performance of applying synchronous and asynchronous training strategies in FL.

\textbf{Virtual Timestamp}. 
Following the best practice in prior FL works~\cite{lai2021fedscale}, we conduct the experiments by simulation while tracking the execution time with virtual timestamps. Specifically, the server begins to broadcast messages containing initial model parameters at timestamp $0$. Then each client sends updates back with a timestamp as the received one plus the execution time of local computation and communication estimated by FedScale~\cite{lai2021fedscale}. 
The server handles the received messages in the order of their timestamps and lets the next broadcast inherit the timestamp from the message that triggers it, assuming the time cost of the server is negligible. Along with an FL course, we record the performance of the global model with respect to such virtual timestamps.

\textbf{Baselines}. 
We implement FedAvg with two synchronous training strategies including \textit{Sync-vanilla} (i.e., the vanilla synchronous strategy) and \textit{Sync-OS} (i.e., the synchronous strategy with over-selection mechanism~\cite{tff}). As Sync-OS is originally proposed and implemented in FedScale~\cite{lai2021fedscale}, we also adapt it for our experiments and report its performance (denoted as \textit{Sync-OS (FedScale)}) for correctness verification.

For asynchronous FL, we instantiate different asynchronous behaviors discussed in Section~\ref{subsubsec:async_behavior}, and different strategies are named in the format of \textit{Async-AdoptedEvent-BroadcastManner-SampleStrategy}.
For example, \textit{Async-Goal-Rece-Unif} denotes that this strategy adopts the event ``{\sf goal\_achieved}'', the \textit{after receiving} broadcasting manner and the uniform sampling strategies for asynchronous FL, which can be regarded as the implementation of FedBuff~\cite{nguyen2022fedbuff}; and \textit{Async-Time-Aggr-Group} denotes we adopt the event ``{\sf time\_up}'', the \textit{after aggregating} broadcasting manner and a group sampling strategy (the client would be grouply sampled according to their responsiveness~\cite{chai2021fedat}).

\textbf{Analysis}.
We adopt the virtual time cost (hours) to achieve the targeted test accuracy as the performance metric for comparing synchronous and asynchronous FL.
The experimental results are shown in Table~\ref{table:main_comparison} (more experimental results can be found in Appendix~\ref{appendix:iid_distribution} and \ref{appendix:more_asyn_results}), from which we can observe that asynchronous training strategies achieve significant efficiency improvements (5.25$\times$\textasciitilde 18.82$\times$) compared to the vanilla synchronous training strategy on all the benchmark datasets.
Meanwhile, we plot the learning curves in Figure~\ref{fig:async_results}. Due to the space limitation, we only show some asynchronous training strategies on CIFAR-10 dataset and omit other similar results. From Figure~\ref{fig:async_results}, we can observe the existence of noticeable gaps between synchronous and asynchronous training strategies for a long time during the training process.
These experimental results are consistent with previous studies~\cite{huba2022papaya, xie2019asynchronous} and confirm that the asynchronous training strategies provided in \ours can significantly improve the training efficiency while achieving competitive model performance.

Both our implementation \textit{Sync-OS} and the original implementation in FedScale show that applying over-selection mechanism in synchronous FL can improve the efficiency to some degree. However, it might cause unfairness among participants and then lead to model bias, as demonstrated in Figure~\ref{fig:agg_count}. From the figure we can observe that when applying over-selection mechanism \textit{Sync-OS}, some clients never contribute to the federated aggregation, i.e., $\Pr[\text{effective\_aggregation\_count}=0]>0$.
The reason is that these clients need more computation or communication time, and thus their feedback would always be dropped since the server has finished the federated aggregation with the feedback from those clients having faster response speeds.
In other words, these clients always become the victims among the over-selected clients, which results in unfairness among participants, and then causes the learned models to bias towards those clients with fast response speeds.
In contrast, the asynchronous learning strategies provided in \ours can improve the efficiency without introducing such unfairness and model bias, due to the fact that staled feedback would be tolerated in the federated aggregation. Hence the distribution of effective aggregation count of asynchronous learning strategies plotted in Figure~\ref{fig:agg_count} is more concentrated and similar to that of the vanilla synchronous training strategy.

\begin{figure*}
    \begin{minipage}[t]{0.33\textwidth}
    \setcaptionwidth{0.96\textwidth}
    \centering
    \includegraphics[width=0.96\textwidth]{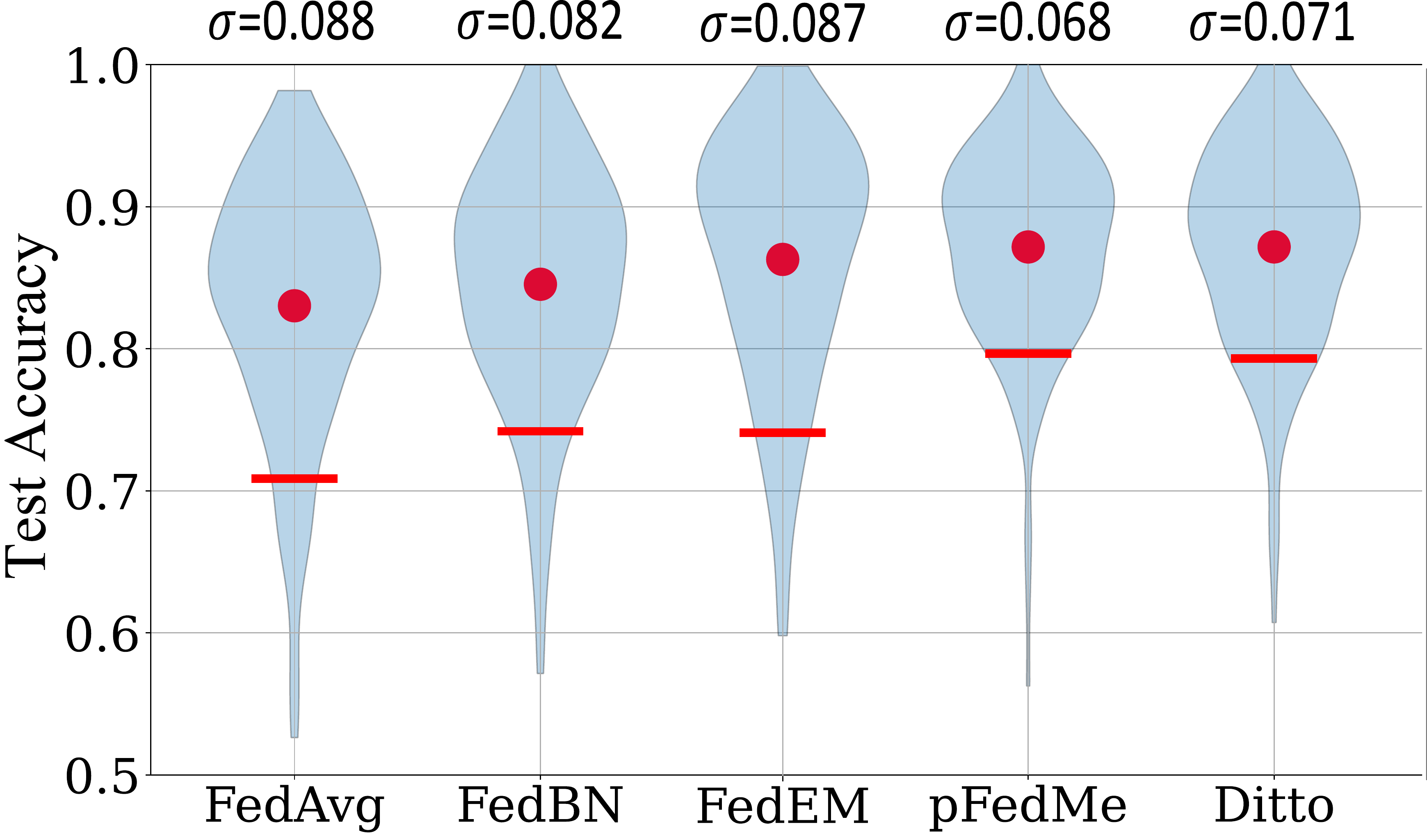}
    \vspace{-0.1in}
    \caption{Client-wise test accuracy on FEMNIST dataset.}
    \label{fig:pfl}
    \end{minipage}
     \begin{minipage}[t]{0.33\textwidth}
     \setcaptionwidth{0.96\textwidth}
    \centering
    \includegraphics[width=0.96\textwidth]{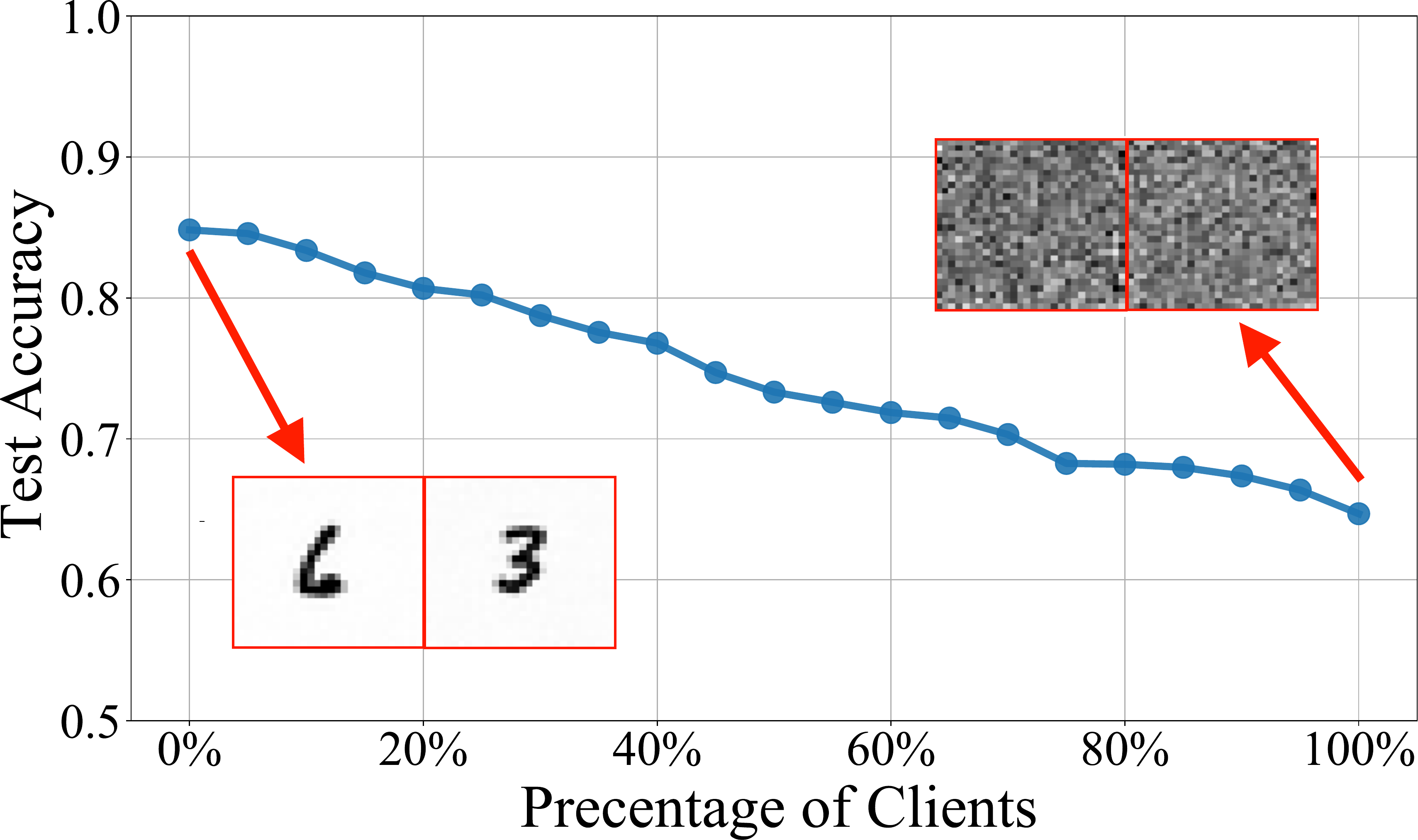}
    \vspace{-0.1in}
    \caption{Accuracy w.r.t. varying protection strength and recovered images.}
    \label{fig:dp}
    \end{minipage}
     \begin{minipage}[t]{0.33\textwidth}
     \setcaptionwidth{0.96\textwidth}
    \centering
    \includegraphics[width=0.96\textwidth]{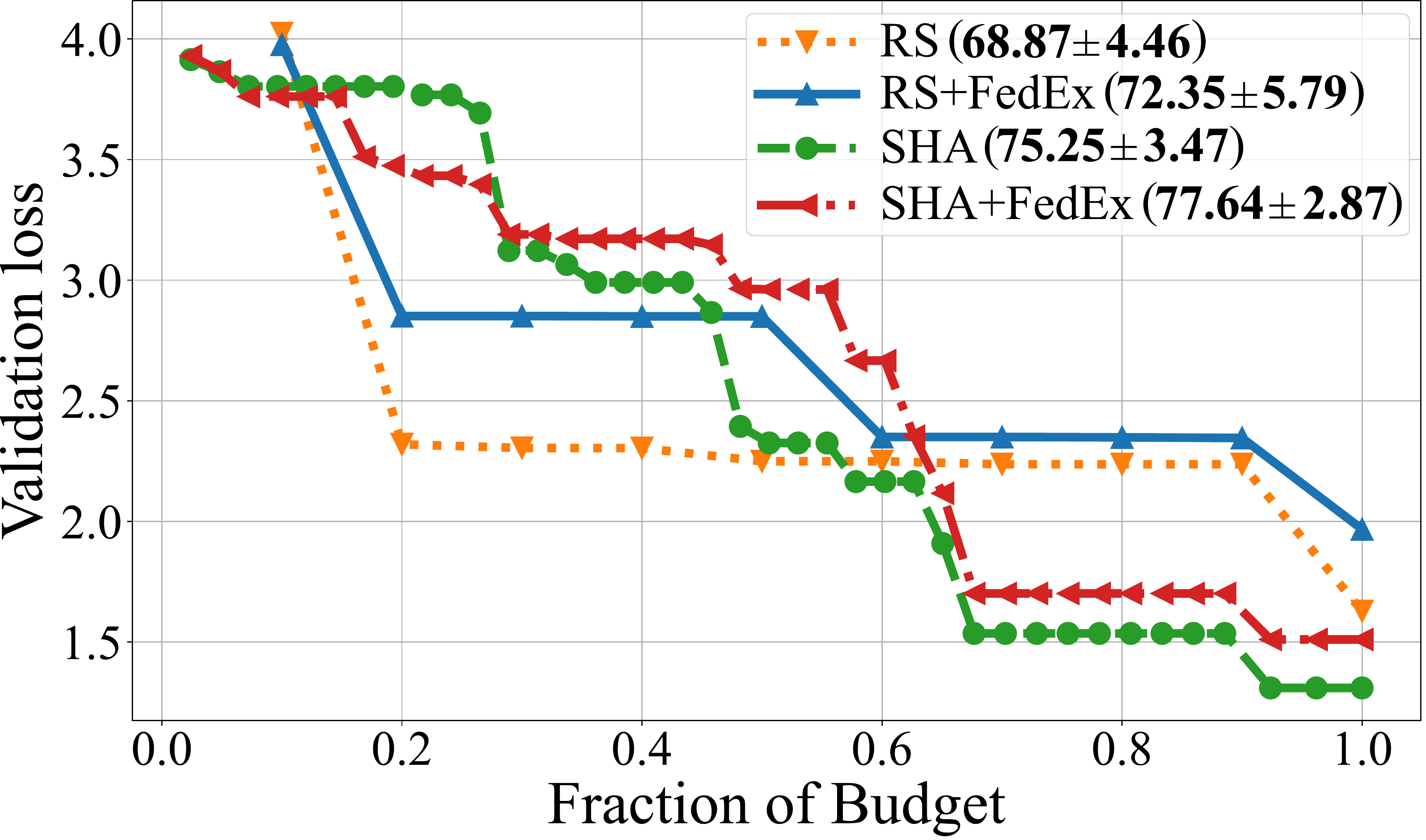}
    \vspace{-0.1in}
    \caption{Best-seen validation loss over time on FEMNIST dataset.}
    \label{fig:hpo}
    \end{minipage}
\end{figure*}

Further, in Figure~\ref{fig:staleness}, we illustrate the characteristics of different asynchronous training strategies in terms of staleness (i.e., the version difference between the up-to-date global model and the model used for local training) of the updates when performing federated aggregation. 
By comparing \textit{Async-Goal-Aggr-Unif} and \textit{Async-Goal-Rece-Unif}, we can see that \textit{after aggregating} broadcasting manner causes less staleness than \textit{after receiving}. 
It implies that \textit{after aggregating} is more suitable for those FL tasks with a low staleness toleration threshold, but such a broadcasting manner requires more available bandwidths at the server since multiple messages are sent out at the same time.

\subsubsection{Personalization}
To demonstrate how personalization can handle the heterogeneity among participants,
we compare FedAvg~\cite{mcmahan2017communication} with several built-in SOTA personalized FL algorithms, including FedBN~\cite{fedbn}, FedEM~\cite{fedem}, pFedMe~\cite{pFedME} and Ditto~\cite{li2021ditto}.

The experimental results are illustrated in Figure~\ref{fig:pfl}, which shows the client-wise test accuracies on FEMNIST dataset.
We can observe that the average accuracy (denoted as the red dots) and the 90\% quantile accuracy (denoted as the red horizontal lines) of vanilla FedAvg are both significantly lower than those of personalized FL algorithms. 
This indicates that applying personalized FL algorithms can improve the client-wise performance, also covering the bottom clients, and then lead to a better overall performance.
Besides, in terms of the standard deviation among the client-wise accuracy (shown as $\sigma$ at the top of the figure), personalized FL algorithms can reduce the performance differences to some degree, which confirms the advantages of enabling personalization behaviors for handling the heterogeneity among participants in real-world FL applications.

Personalized federated learning algorithms might need different computation and computation resources compared to vanilla FedAvg~\cite{mcmahan2017communication}.
The computation and communication costs in each training round are determined by the adopted algorithms.
Take the comparisons between vanilla FedAvg and two representative Personalized federated learning algorithms FedBN~\cite{fedbn} and Ditto~\cite{li2021ditto} as examples, in each training round~\cite{chen2022pfl}, (i) FedBN needs the same computation but fewer communication costs, since it proposes to not share the parameters of BacthNorm layer; and (ii) Ditto needs the same communication but more local computation costs, since it suggests to train the local models additionally. 
Further, from the perspective of an FL course, i.e., iteratively performing the FL training rounds until termination, the communication and computation costs depend on the convergence of learned models.

\subsubsection{Privacy Protection and Attack}
\label{sec:privacy}
We conduct an experiment to show the effect of applying privacy protection algorithms provided in \ours.
We take DP as an example, and study its effect on the utility of learned model and the effectiveness in defending against privacy attack.
Specifically, we train a ConvNet2 model on FEMNIST, and randomly choose some of the clients to inject Gaussian noise into the returned model updates to strengthen their privacy.
We construct multiple FL courses with respect to varying the percentage of clients that injects noise, changing from 0\% to 100\%, and plot the performance of the learned models in Figure~\ref{fig:dp}.
From this figure we can observe that as more and more clients choose to inject the noise into the returned model updates, the test accuracy achieved by the learned global model decreases gradually, from 84\% to 65\%, which shows the trade-off between the privacy protection strength and model utility.

Moreover, we apply the DLG algorithm~\cite{zhu2019dlg} implemented in \ours to conduct privacy attack, aiming to reconstruct private training data of other users.
As shown on the left-hand side of Figure~\ref{fig:dp}, the reconstructed images from the clients who have not injected noises are clear and the privacy attacker successfully recovers clients' training data to the extent that the groundtruth digits are exposed.
On the right-hand side of the figure, we plot the reconstructed images from those clients injecting noises, which confirms the effectiveness of the privacy protection provided by DP since the attacker fails to recover meaningful information.

\subsubsection{Auto-Tuning}
As mentioned in Section~\ref{subsec:auto}, we have implemented several HPO methods in \ours, which enables users to auto-tune hyperparameters of FL courses. Here we experimentally compare some representative HPO methods, including random search (RS)~\cite{bergstra2012random}, successive halving algorithm (SHA)~\cite{li2017hyperband} and recently proposed Federated-HPO method FedEx~\cite{fedex}, by applying them to optimize hyperparameters of FedAvg on FEMNIST dataset.
We follow the protocol used in FedHPO-B~\cite{wang2022fedhpo}, where RS and SHA try configurations one by one, and FedEx wrapped by RS/SHA manages the search procedure in a fine-grained granularity to explore hyperparameter space concurrently.

We present the results in Figure~\ref{fig:hpo}, where the best-seen validation loss is depicted, and the test accuracy of the searched optimal configuration is reported in the legend.
The best-seen validation losses of wrapped FedEx decrease slower than their corresponding wrappers, where such a poorer regret seems to indicate poorer searched hyperparameter configurations. However, their searched configurations' test accuracies are remarkably better than their wrappers, implying the superiority of managing the search procedure in a fine-grained granularity.

%% file: subsections/6_conclusion.tex
\section{Conclusions}
\label{sec:conclu}
In this paper, we introduce \ours, a novel federated learning platform, to provide users with great supports for various FL development and deployment.
Towards both convenient usage and flexible customization, \ours exploits an event-driven architecture to frame an FL course into $<$events, handlers$>$ pairs so that users can describe participants' behaviors from their respective perspectives.
Such an event-driven design makes \ours suitable for handling various types of heterogeneity in FL, due to the advantages that (i) \ours enables participants to exchange rich types of messages, express diverse training behaviors, and optimize different learning goals, and (ii) \ours offers rich condition checking events to support various coordinations and corporations among participants, such as different asynchronous training strategies.
Further, the design of \ours allows us to conveniently implement and provide several important plug-in components, such as privacy protection, attack simulation, and auto-tuning, which are indispensable for practical usage.
We have released \ours to help researchers and developers quickly get started, develop new FL algorithms, and build new FL applications, with the goal of promoting and accelerating the progress of FL.

%% file: subsections/Appendix.tex
\section{Convergence Analysis}
\label{appendix:proof}
Without loss of generality, we assume the numbers of training instances are same among clients and simplify Eq.~\eqref{eq:loss} as the following optimization problem:
\begin{equation}
\label{eq:objective}
    \min_{\theta \in \mathbb{R}^d} F(\theta):= \frac{1}{M}\sum_{i=1}^{M} F_i(\theta),
\end{equation}
where $M$ is the number of participated clients, $F_i$ denotes the loss function of client $i$.
We use $\mathcal{S}^{(t)} \subseteq [M]$ to denote the index set of the clients that participate in the $t$-th training round. For each client $i\in\mathcal{S}^{(t)}$, it takes $Q$ local SGD steps with learning rate $\eta$. The model update at the $t$-th round is:
\begin{align}
     \theta^{(t+1)} - \theta^{(t)} & = \Delta^{(t)} \nonumber\\
    & = \frac{1}{|\mathcal{S}^{(t)}|}\sum_{i\in \mathcal{S}^{(t)}} \Delta_i^{(t)} \nonumber\\
    & = -\frac{1}{|\mathcal{S}^{(t)}|}\sum_{i\in \mathcal{S}^{(t)}} \sum_{q=1}^{Q} \eta g_i(\theta^{(t-\tau_i)}_{i,q}), \label{eq:theta_delta}
\end{align}
where $g_i(\cdot)$ denotes stochastic gradient of client $i$. 
We use $\tau_i(t)$ to denote the version staleness of initialization model that is used for local training in client $i$ at $t$-th training round. And we simplify it as $\tau_i$ when it would not cause any ambiguity.

Firstly, we given some widely-adopted assumptions~\cite{nguyen2022fedbuff, chai2021fedat, wang2021field}:

\begin{assumption}[Smoothness]
\label{assum:smoothness}
The loss function $F$ has Lipschitz continuous gradients with a constant $L>0$, i.e., $\forall \theta_1, \theta_2$:
\begin{equation}
    ||\nabla F(\theta_1) - \nabla F(\theta_2)||^2 \leq L^2 ||\theta_1 - \theta_2 ||^2,
\end{equation}
and further:
\begin{equation}
    F(\theta_1) - F(\theta_2) \leq \langle \nabla F(\theta_2), \theta_1-\theta_2 \rangle + \frac{L}{2}|| \theta_1 - \theta_2 ||^2.
\end{equation}
\end{assumption}

\begin{assumption}[Convexity]
\label{assum:convexity}
The loss function $F$ is $\mu$-strongly convex with a constant $\mu>0$, i.e., $\forall \theta_1, \theta_2$:
\begin{equation}
    F(\theta_1) - F(\theta_2) \geq \langle \nabla F(\theta_2), \theta_1-\theta_2 \rangle + \frac{\mu}{2}|| \theta_1 - \theta_2 ||^2.
\end{equation}
\end{assumption}

\begin{assumption}[Unbiasedness]
\label{assm:unbiased}
The estimation of stochastic gradient is unbiased: $\mathbb{E}[{g_i(\theta)}] = \nabla F_i(\theta)$.
\end{assumption}

\begin{assumption}[Bounded variances]
\label{assm:bounded_var}
The local and global variances are bounded for all clients, i.e., $\forall i$, there exist constants $\sigma_l$ and $\sigma_g$, s.t., 
\begin{align}
    \mathbb{E}\left[||g_i(\theta) - \nabla F_i(\theta)||^2\right] \leq \sigma_l^2,\\
    \frac{1}{M}\sum_{i=1}^{M}||\nabla F_i(\theta) - \nabla F(\theta) ||^2 \leq \sigma_g^2. 
\end{align}
\end{assumption}

\begin{assumption}[Bounded gradients]
\label{assm:bounded_grad}
The gradients of clients are bounded, i.e., $||\nabla F_i ||^2 \leq C, \forall i\in [M]$.
Specifically, $||\nabla F||^2 = ||\frac{1}{M}\sum_{i=1}^{M}\nabla F_i ||^2 \leq \frac{1}{M}\sum_{i=1}^{M}||\nabla F_i||^2 \leq C$. 
\end{assumption}

\begin{assumption}[Bounded staleness]
\label{ass:bounded_staleness}
$\forall i \in [M]$, the staleness $\tau_i$ is not larger than $\tau_{\max}$.
\end{assumption}

Based on these assumptions and inspired by previous studies~\cite{chai2021fedat, nguyen2022fedbuff}, we have two lemmas:
\begin{lemma}
\label{lemma:bound_gradient_l2_norm}
The expectation of the $L_2$ norm of the estimated gradients for all clients are bounded:
\begin{align}
    \mathbb{E}\left[||g_i(\theta^{(t)})||^2\right] \leq 3\left(\sigma_l^2 + \sigma_g^2 + C\right).
\end{align}
\end{lemma}
\begin{proof}
\begin{align}
     &\thinspace\mathbb{E}\left[||g_i(\theta^{(t)})||^2\right] \nonumber\\
     = &\thinspace\mathbb{E}\left[|| g_i(\theta^{(t)}) - \nabla F_i(\theta^{(t)}) +  \nabla F_i(\theta^{(t)}) - \nabla F(\theta^{(t)}) + \nabla F(\theta^{(t)}) ||^2 \right] \nonumber\\
     \leq& \thinspace3\mathbb{E}\left[||g_i(\theta^{(t)}) - \nabla F_i(\theta^{(t)})||^2 +  ||\nabla F_i(\theta^{(t)}) - \nabla F(\theta^{(t)})||^2 \right. \nonumber\\
     & \quad \:\left.+ ||\nabla F(\theta^{(t)})||^2 \right] \nonumber\\
     =& \thinspace 3\left(\sigma_l^2 + \sigma_g^2 + C\right),
\end{align}
where the last equality follows Assumption~\ref{assm:bounded_var} and~\ref{assm:bounded_grad}.
\end{proof}

\begin{lemma}
\label{lemma:bound_delta_l2_norm}
The expectation of $L_2$ norm of $\Delta^{(t)}, \forall t \in [T]$, is bounded:
\begin{align}
    \mathbb{E}\left[||\Delta^{(t)}||^2\right] \leq 3 Q^2 \eta^2 \left(\sigma_l^2 + \sigma_g^2 + C\right),
\end{align}
where $Q$ denotes the local SGD steps and $\eta$ denotes the learning rate.
\end{lemma}
\begin{proof}
According to Eq.~\eqref{eq:theta_delta}, we have:
\begin{align}
    ||\Delta^{(t)}||^2 
    & = \frac{1}{|\mathcal{S}^{(t)}|^2} \Big|\Big|\sum_{i\in \mathcal{S}^{(t)}} \sum_{q=1}^{Q}\eta g_i(\theta_{i,q}^{t-\tau_i}) \Big|\Big|^2 \nonumber\\
    & \leq \frac{Q\eta}{|\mathcal{S}^{(t)}|} \sum_{i\in \mathcal{S}^{(t)}} \sum_{q=1}^{Q} \big|\big| g_i(\theta_{i,q}^{t-\tau_i}) \big|\big|^2.
\end{align}
Take the expectation over the randomness w.r.t. the client participation and the stochastic gradients:
\begin{align}
    \mathbb{E}\big[||\Delta^{(t)}||^2\big] 
    & \leq \frac{Q \eta^2}{|\mathcal{S}^{(t)}|} \sum_{i\in \mathcal{S}^{(t)}} \sum_{q=1}^{Q}  \mathbb{E}\left[|| g_i(\theta_{i,q}^{t-\tau_i}) ||^2\right] \nonumber\\
    & \leq 3 Q^2 \eta^2 \left(\sigma_l^2 + \sigma_g^2 + C\right),
\end{align}
where the last inequality follows from Lemma~\ref{lemma:bound_gradient_l2_norm}
\end{proof}

Next we provide the convergence analysis for the proposed asynchronous training protocol in federated learning, inspired by previous studies~\cite{chai2021fedat, nguyen2022fedbuff, wang2021field}.
With Eq.~\eqref{eq:theta_delta} and the assumption of smoothness (Assumption~\ref{assum:smoothness}), we have:
\begin{align}
    &F(\theta^{(t+1)}) - F(\theta^{(t)}) \nonumber \\
    &\leq \langle \nabla F(\theta^{(t)}), \theta^{(t+1)} - \theta^{(t)} \rangle + \frac{L}{2} || \theta^{(t+1)} - \theta^{(t)} ||^2 \nonumber \\
    &= \langle \nabla F(\theta^{(t)}), \Delta^{(t)} \rangle + \frac{L}{2} || \Delta^{(t)} ||^2.
\end{align}
Take the total expectation:
\begin{align}
    & \quad \mathbb{E}\big[F(\theta^{(t+1)}) - F(\theta^{(t)})\big] \nonumber \\
    & \leq \mathbb{E}\big[\langle \nabla F(\theta^{(t)}), \Delta^{(t)} \rangle\big] + \frac{L}{2} \mathbb{E}\big[|| \Delta^{(t)} ||^2\big] \nonumber \\
    & \leq \mathbb{E}\big[\underbrace{\langle \nabla F(\theta^{(t)}), \Delta^{(t)} \rangle}_{H_1} \big] + \frac{3LQ^2 \eta^2}{2} \left(\sigma_l^2+\sigma_g^2 + C\right) \quad (\text{With Lemma}~\ref{lemma:bound_delta_l2_norm}). 
    \label{eq:object}
\end{align}
Consider the bound of $H_1$:
\begin{align}
    H_1 &= \langle \nabla F(\theta^{(t)}), -\frac{1}{|\mathcal{S}^{(t)}|}\sum_{i\in \mathcal{S}^{(t)}}\sum_{q=1}^{Q}\eta g(\theta_{i,q}^{(t-\tau_i)}) \rangle \nonumber\\
    & = -\frac{1}{|\mathcal{S}_t|}\sum_{i\in \mathcal{S}^{(t)}}\sum_{q=1}^{Q}\eta \langle \nabla F(\theta^{(t)}), g(\theta_{i,q}^{(t-\tau_i)}) \rangle.
\end{align}
We take the total expectation $\mathbb{E}[\cdot]:=\mathbb{E}_{\mathcal{F}}\mathbb{E}_{i\sim [M]}\mathbb{E}_{g_i|i,\mathcal{F}}[\cdot]$ ($\mathcal{F}$ denotes the historical information):
\begin{align}
    \mathbb{E}[H_1] &= -\mathbb{E}_{\mathcal{F}}\left[\frac{1}{M}\sum_{i=1}^{ M}\sum_{q=1}^{Q} \eta \mathbb{E}_{g_i|i\sim [M]}\big[\langle \nabla F(\theta^{(t)}), g(\theta_{i,q}^{(t-\tau_i)}) \rangle \big]  \right] \nonumber \\
    &= -\mathbb{E}_{\mathcal{F}}\left[\sum_{q=1}^{Q}\eta \langle\nabla F(\theta^{(t)}), \frac{1}{M}\sum_{i=1}^{M}\nabla F_i(\theta_{i,q}^{(t-\tau_i)}) \rangle  \right].
\end{align}
Given $\langle a,b \rangle = \frac{1}{2}(||a||^2 + ||b||^2 - ||a-b||^2)$, we have:
\begin{align}
    \mathbb{E}[H_1] &=  \sum_{q=1}^{Q}\frac{\eta}{2} \mathbb{E}_{\mathcal{F}}\Big[ -||\nabla F(\theta^{(t)})||^2 - ||\frac{1}{M}\sum_{i=1}^{M} \nabla F_i(\theta_{i,q}^{(t-\tau_i)})||^2  \nonumber \\
    &\quad  + ||\nabla F(\theta^{(t)})-\frac{1}{M}\sum_{i=1}^{M} \nabla F_i(\theta_{i,q}^{(t-\tau_i)})||^2 \Big] \nonumber \\
    & \leq -\mu Q \eta\mathbb{E}_{\mathcal{F}}\left[ F(\theta^{(t)}) - F(\theta^{(*)})\right] \nonumber\\
    & \quad + \sum_{q=1}^{Q}\frac{\eta}{2}\mathbb{E}_{\mathcal{F}}\Big[\underbrace{||\nabla F(\theta^{(t)}) - \frac{1}{M}\sum_{i=1}^{M} \nabla F_i(\theta_{i,q}^{(t-\tau_i)})||^2}_{H_2} \Big],
    \label{eq:h1}
\end{align}
where the last inequality follows from Assumption~\ref{assum:convexity} and $\theta^{(*)}$ denotes the optimum of minimizing $F(\cdot)$.
Further, with Eq.~\eqref{eq:objective}, we have:
\begin{align}
    H_2 &= \Big|\Big|\frac{1}{M}\sum_{i=1}^{M}\nabla F_i(\theta^{(t)}) - \frac{1}{M}\sum_{i=1}^{M} \nabla F_i(\theta_{i,q}^{(t-\tau_i)}) \Big|\Big|^2 \nonumber\\
    & \leq \frac{1}{M}\sum_{i=1}^{M} \big|\big| \nabla F_i(\theta^{(t)}) - \nabla F_i(\theta_{i,q}^{(t-\tau_i)}) \big|\big|^2 \nonumber\\
    & \leq \frac{2}{M}\sum_{i=1}^{M} \Big[\big|\big| \nabla F_i(\theta^{(t)}) - \nabla F_i(\theta^{(t-\tau_i)}) \big|\big|^2  + \big|\big| \nabla F_i(\theta^{(t-\tau_i)}) - \nabla F_i(\theta_{i,q}^{(t-\tau_i)}) \big|\big|^2\Big] \nonumber\\
    & \leq \frac{2L^2}{M}\sum_{i=1}^{M} \left[\big|\big| \theta^{(t)} - \theta^{(t-\tau_i)} \big|\big|^2 + \big|\big| \theta^{(t-\tau_i)} - \theta_{i,q}^{(t-\tau_i)} \big|\big|^2 \right].
\end{align}
Take the expectation, we have:
\begin{align}
    \mathbb{E}[H_2] &\leq \frac{2L^2}{M}\sum_{i=1}^{M} \Big[ \mathbb{E}\big[\big|\big|\theta^{(t)} - \theta^{(t-\tau_i)}\big|\big|^2\big] + \mathbb{E}\big[\big|\big|\theta^{(t-\tau_i)} - \theta_{i,q}^{(t-\tau_i)}\big|\big|^2\big] \Big] \nonumber\\
    & \leq \frac{2L^2}{M}\sum_{i=1}^{M} \Big[ \mathbb{E}\big[\big|\big|\sum_{\rho=t-\tau_i}^{t-1}\frac{1}{|\mathcal{S}_{\rho}|}\sum_{i\in \mathcal{S}_{\rho}}\sum_{q=1}^{Q}\eta g_i(\theta_{i,q}^{(\rho)})\big|\big|^2 \big] \nonumber\\
    & \qquad + \mathbb{E}\big[\big|\big| \sum_{q=1}^{Q}\eta g_i(\theta_{i,q}^{(t-\tau_i)}) \big|\big|^2\big] \Big] \nonumber\\
    & \leq \frac{2L^2}{M}\sum_{i=1}^{M} \Big[ Q \eta^2 \tau_i \sum_{\rho=t-\tau_i}^{t-1}\frac{1}{|\mathcal{S}_{\rho}|}\sum_{i\in \mathcal{S}_{\rho}} \sum_{q=1}^{Q} \mathbb{E} \big[||g_i(\theta_{i,q}^{(\rho)})||^2\big] \nonumber\\
    & \qquad + Q\eta^2\sum_{q=1}^{Q} \mathbb{E}\big[||g_i(\theta_{i,q}^{(t-\tau_i)})||^2\big] \Big] \nonumber\\
    & \leq \frac{2L^2}{M}\sum_{i=1}^{M} \Big[ 3Q^2\eta^2\tau_{\max}^2(\sigma_l^2 + \sigma_g^2 + C) + 3Q^2\eta^2(\sigma_l^2 + \sigma_g^2 + C)\Big] \nonumber\\
    & = 6L^2Q^2\eta^2(\tau_{\max}^2+1)(\sigma_l^2 + \sigma_g^2 + C),
    \label{eq:h2}
\end{align}
where the last inequality follows from Lemma~\ref{lemma:bound_gradient_l2_norm} and Assumption~\ref{ass:bounded_staleness}.
By inserting Eq.~\eqref{eq:h2} and Eq.~\eqref{eq:h1} into Eq.~\eqref{eq:object}, we have:
\begin{align}
    &\quad \mathbb{E}\left[F(\theta^{(t+1)}) - F(\theta^{(t)})\right] \nonumber\\
    &\leq -\mu Q \eta\mathbb{E}\left[ F(\theta^{(t)}) - F(\theta^{(*)})\right] \nonumber\\
    &\quad + 3LQ^2\eta^2 \left(\sigma_l^2 + \sigma_g^2 + C\right)\left[\eta Q L (\tau_{\max}^2 +1)+\frac{1}{2}\right].
\end{align}
By rearranging, we have:
\begin{align}
    &\quad \mathbb{E}\big[F(\theta^{(t+1)}) - F(\theta^{(*)})\big] - \frac{3LQ\eta}{\mu} (\sigma_l^2 + \sigma_g^2 + C)[\eta Q L (\tau_{\max}^2 +1)+\frac{1}{2}] \nonumber\\
    &\leq (1-\mu Q \eta)\Big[\mathbb{E}[ F(\theta^{(t)}) - F(\theta^{(*)})]  \nonumber\\
    & \quad - \frac{3LQ\eta}{\mu} (\sigma_l^2 + \sigma_g^2 + C)[\eta Q L (\tau_{\max}^2 +1)+\frac{1}{2}]\Big],
\end{align}
which implies a geometric series with ratio $1-\mu Q \eta$. Thus we can obtain:
\begin{align}
    &\quad \mathbb{E}\big[F(\theta^{(T)}) - F(\theta^{(*)})\big] \nonumber\\
    &\leq (1-\mu Q \eta)^T\mathbb{E}\big[ F(\theta^{(0)}) - F(\theta^{(*)})\big] \nonumber\\
    & \quad +\left[1-(1-\mu Q \eta)^T\right] \frac{3LQ\eta}{\mu} \left(\sigma_l^2 + \sigma_g^2 + C\right)\left[\eta Q L (\tau_{\max}^2 +1)+\frac{1}{2}\right].
\end{align}
When $\mu Q \eta <=1$, we have the following model convergence conclusion:
\begin{align}
    &\quad \mathbb{E}\big[F(\theta^{(T)}) - F(\theta^{(*)})\big] \nonumber\\
    &\leq \left(1-\mu Q \eta\right)^T\mathbb{E}\big[ F(\theta^{(0)}) - F(\theta^{(*)})\big] \nonumber\\
    & \quad +\frac{3LQ\eta}{\mu} \left(\sigma_l^2 + \sigma_g^2 + C\right)\left[\eta Q L (\tau_{\max}^2 +1)+\frac{1}{2}\right].
\end{align}

\begin{table*}[t]
    \centering
    \caption{Examples of events in FederatedScope.}
    \vspace{-0.1in}
    \begin{tabular}{l l l}
        \toprule
         Category & Event & Event Description \\
         \midrule
        \multirow{5}{*}{Related to Message Passing} & {\sf receiving\_join\_in} & The server receives join-in requirements from clients. \\
        & {\sf receiving\_model} & Clients receive the global model from the server. \\
        & {\sf receiving\_updates} & The server receives model updates from clients. \\
        & {\sf receiving\_eval\_request} & Clients receive the request of evaluation from the server. \\
        & ... ... & ... ... \\
        \midrule
        \multirow{5}{*}{Related to Condition Checking} & {\sf all\_received} & All the model updates have been received.\\
        & {\sf time\_up} & The allocated time budget for the training round has run out. \\
        & {\sf early\_stop} & The pre-defined early stop conditions are satisfied. \\
        & {\sf performance\_drop} & The received global model causes a performance drop. \\
        & ... ... & ... ... \\
        \bottomrule
    \end{tabular}
    \label{tab:events}
\end{table*}

\begin{table*}[t]
    \centering
    \caption{The statistics of the datasets provided in DataZoo.}
    \vspace{-0.1in}
    \begin{tabular}{l c c c c c}
        \toprule
        Dataset & Task & Number of Instance &  Number of Clients \\
        \midrule
        FEMNIST & Image Classification & 817,851 & 3,597 \\
        CelebA & Image Classification & 200,288 & 9,323 \\
        CIFAR-10 & Image Classification & 60,000 & 1,000 \\
        \midrule
        Shakespeare &  Next Character Prediction & 4,226,158 & 1,129 \\
        Twitter & Sentiment Analysis & 1,600,498 & 660,120 \\
        Reddit & Language Modeling & 56,587,343 & 1,660,820\\
        \midrule
        DBLP (partitioned by venue) & Node Classification & 52,202 & 20\\
        DBLP (partitioned by publisher) & Node Classification & 52,202 & 8\\
        Ciao & Link Classification & 565,300 & 28 \\
        MultiTask & Graph Classification & 18,661 & 7\\
         \bottomrule
    \end{tabular}
    \label{table:datazoo}
\end{table*}

\section{Examples of events}
\label{appendix:events}
Some examples of events provided in \ours are presented in Table~\ref{tab:events}.
These events and the corresponding handlers are used to describe participants' behaviors in \ours, which is introduced in Section~\ref{subsec:event_driven}.

\section{DataZoo}
\label{appendix:datazoo}
The statistics of the datasets provided in DataZoo are summarized in Table~\ref{table:datazoo}. Our DataZoo contains ten widely-used datasets collected from various FL applications and standardizes the data preprocessing.

\section{Overall Structure}
\label{appendix:overvall_architecture}
The overall structure is illustrated in Figure~\ref{fig:overall}.
\begin{figure}[t]
    \centering
    \includegraphics[width=0.45\textwidth]{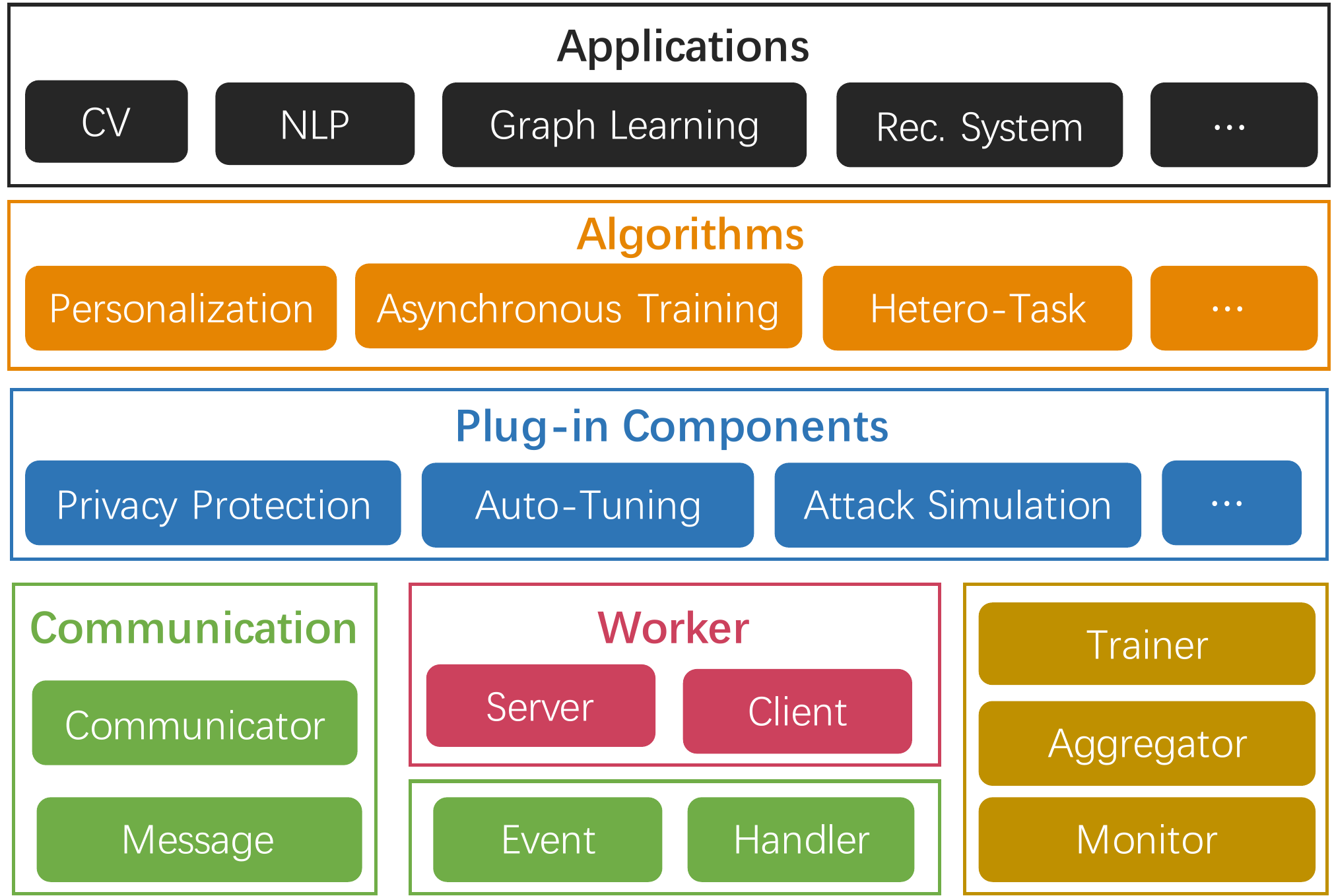}
    \vspace{-0.1in}
    \caption{The overall structure of \ours.}
    \label{fig:overall}
\end{figure}

\section{Completeness Checking}
\label{appendix:completeness}
As shown in Figure~\ref{fig:completeness}, \ours provides a completeness checking to verify the flow of message transmission in the constructed FL courses. A complete FL course should contain at least one path from the ``start'' node to the ``termination'' node.

\begin{figure*}[t]
    \centering
    \includegraphics[width=0.85\textwidth]{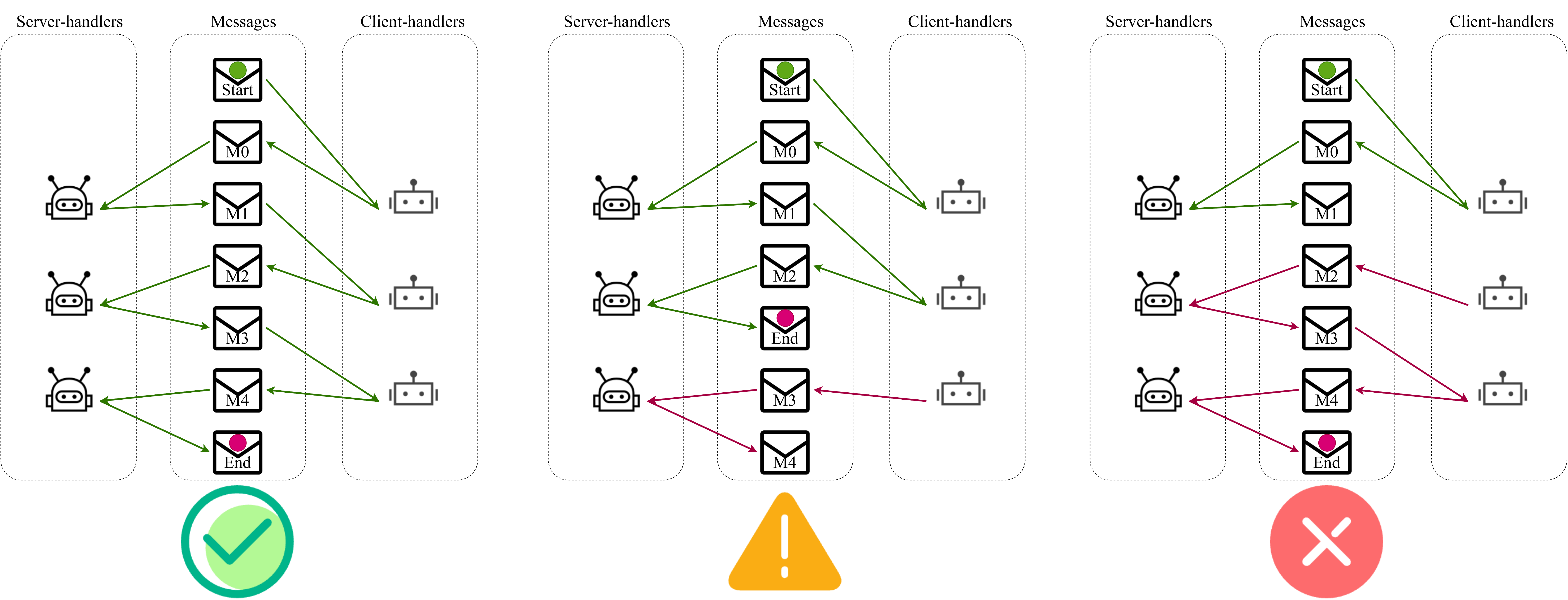}
    \caption{The graphs to verify the flow of message transmission in a constructed FL course. In  the left and middle subgraphs, there exists at least one path from the ``start'' node to the ``termination'' node, which indicates the FL courses are complete. In the middle subgraph, the nodes that are unreachable from the start node, .e.g, ``M3" and ``M4", are redundant, and \ours would provide warnings for these redundant nodes. The subgraph on the right side denotes an incomplete FL course since no start-to-end path, and \ours would raise an error in the completeness checking to help users. }
    \label{fig:completeness}
\end{figure*}

\section{Implementation Details}
\label{appendix:implementation}
We conduct the experiments show in Section~\ref{sec:exp} on 8 GeForce RTX 2080Ti GPUs.

For the experiments on FEMNIST and CIFAR-10, we train a CNN with two convolutional layers. 
We set the hidden size to 2048 and the dropout ratio to 0.5 to prevent overfitting. At each round of training, the clients execute 4 SGD steps with a batch size of 20, and the learning rate is set to 0.5. The concurrency number (i.e., the number of clients that perform local training at each training round) is 100.
For the experiments on Twitter, we set the embedding size to 300 and the maximum length to 140 to embed the texts with a bag-of-words model, and train a logistic regression model. At each round of training, the clients execute 4 SGD steps with a batch size of 2, and the learning rate is set to 0.05. The concurrency number is set to 200.

For all the implemented synchronous and asynchronous strategies, we tune the hyperparameter on the validation set.
For Sync-OS, at each training round, we over-select 30\% more clients upon the corresponding concurrency number, i.e., 130 on FEMNIST and CIFAR-10, and 260 on Twitter. For asynchronous settings, we set the aggregation goal to 40, 20, and 40 for the experiments conducted on FEMNIST, CIFAR-10, and Twitter, respectively. And the threshold of staleness toleration is set to 20 for all the datasets. For the asynchronous strategies equipped with "{\sf time\_up}", the time budget of each training round is set to the same value as the averaged time cost for achieving the defined aggregation goal when using "{\sf goal\_achieved}".

\section{datasets with IID distribution versus non-IID distribution}
\label{appendix:iid_distribution}
We split CIFAR-10 into 100 clients, following the uniform distribution and the Dirichlet distribution (the Dirichlet factors $\alpha$ are set to 1.0, 0.5, and 0.2, a smaller factor value implies a higher heterogeneous degree) to synthesize distributed datasets with IID distribution and non-IID distribution, respectively. Then we federally train a CNN with two convolutional layers, adopting vanilla FedAvg~\cite{mcmahan2017communication} and two representative personalized federated learning algorithms, i.e., FedBN~\cite{fedbn} and Ditto~\cite{li2021ditto}. The experimental results are demonstrated in Table~\ref{tab:iid_vs_noniid}.

\begin{table}[t]
    \centering
    \caption{Experimental results (accuracy) on CIFAR-10 with IID distribution and non-IID distributions.}
    \begin{tabular}{c c c c c}
    \toprule
    \multirow{2}{*}{Methods} & \multirow{2}{*}{IID Distribution} & \multicolumn{3}{c}{Non-IID Distribution}\\
    \cmidrule(lr){3-5}
    &  & \multicolumn{1}{c}{$\alpha=1.0$} & \multicolumn{1}{c}{$\alpha=0.5$} & \multicolumn{1}{c}{$\alpha=0.2$} \\
    \midrule
    FedAvg & 0.8049 & 0..7929 & 0.7905 & 0.7700 \\
    FedBN & 0.7908 & 0.8106 & 0.8311 & 0.8817  \\
    Ditto & 0.7708 & 0.8087 & 0.8278 & 0.8840 \\
    \bottomrule
    \end{tabular}
    \label{tab:iid_vs_noniid}
\end{table}

From the experimental results we can observe that, although vanilla FedAvg can achieve competitive performance on CIFAR-10 with IID distribution, it cannot perform well on non-IID distribution datasets and has larger performance drops when the heterogeneity degree increases. The heterogeneity in clients' data, which is widespread in FL applications, could lead to sub-optimal performance when participants learn a simple global model as what they do in distributed machine learning. Therefore, to handle such heterogeneity, clients are encouraged to apply client-specific training based on their private data, and share parts of the global model while locally maintaining others. As a result, FedBN and Ditto obtain noticeable improvement on datasets with non-IID distributions compared to FedAvg, which indicates their effectiveness in handling such heterogeneity of FL.

\begin{figure*}
    \begin{minipage}[t]{0.33\textwidth}
    \centering
    \includegraphics[width=0.96\textwidth]{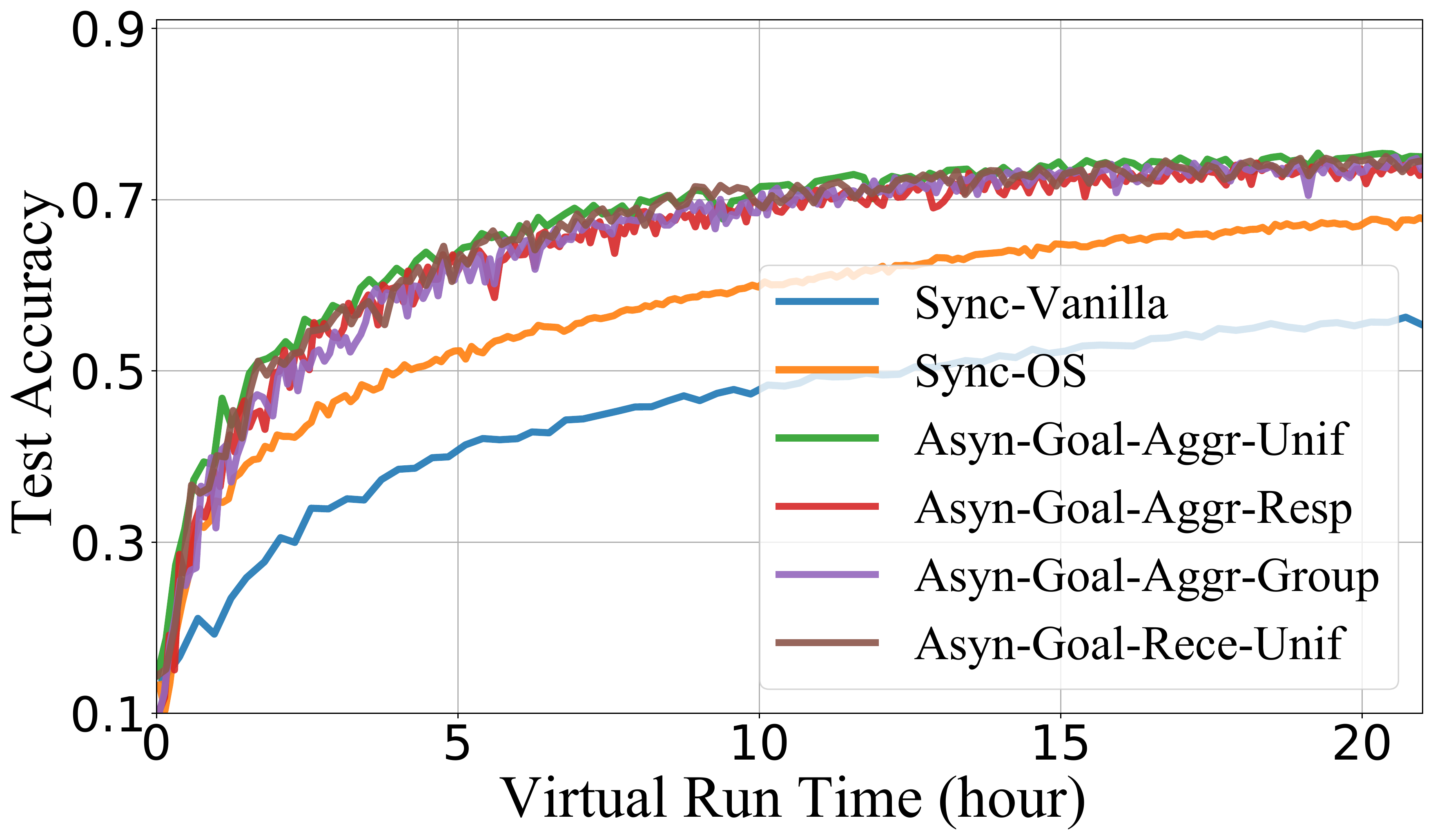}
    \vspace{-0.05in}
    \end{minipage}
     \begin{minipage}[t]{0.33\textwidth}
    \centering
    \includegraphics[width=0.96\textwidth]{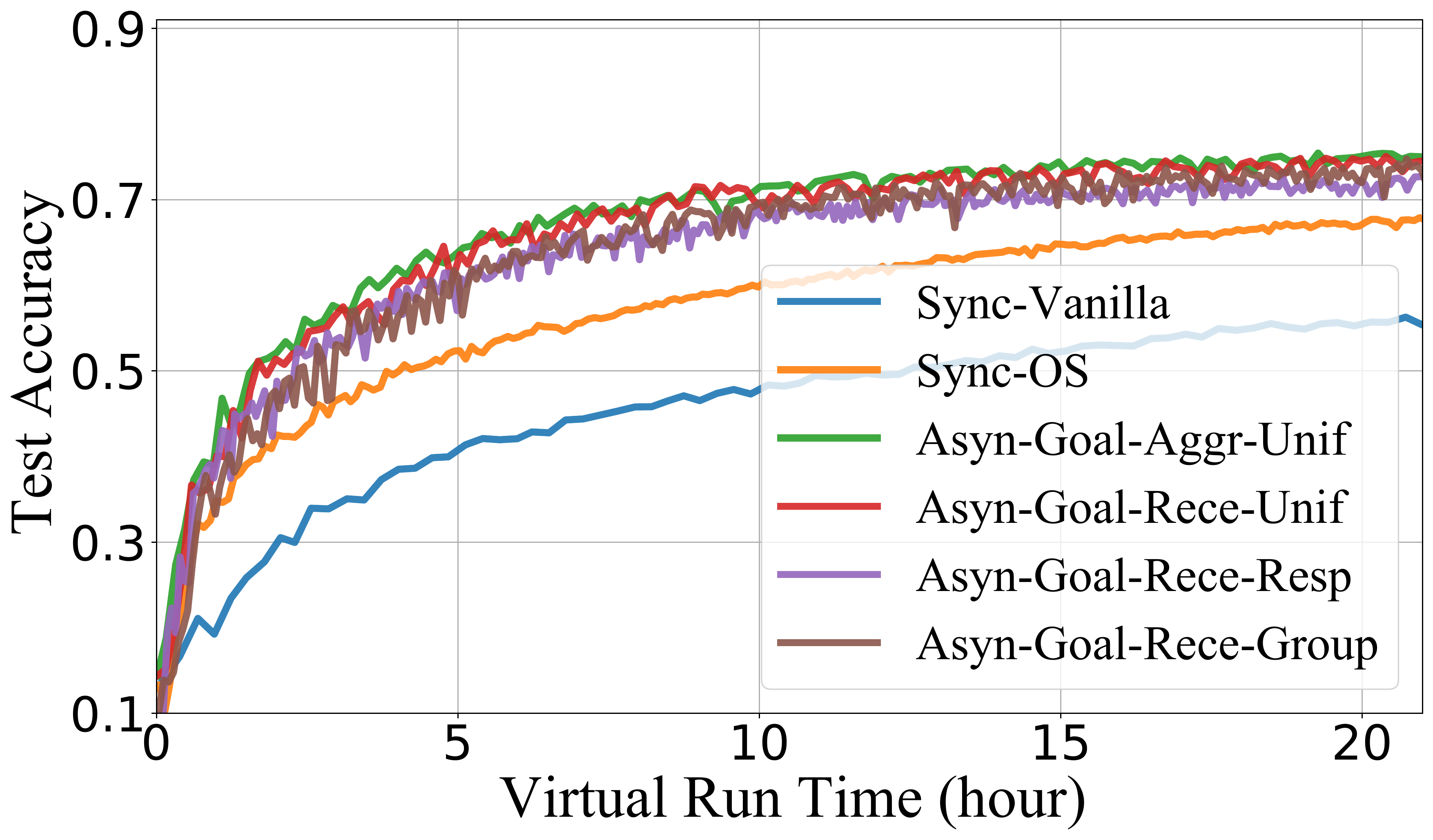}
    \vspace{-0.05in}
    \end{minipage}
     \begin{minipage}[t]{0.33\textwidth}
    \centering
    \includegraphics[width=0.96\textwidth]{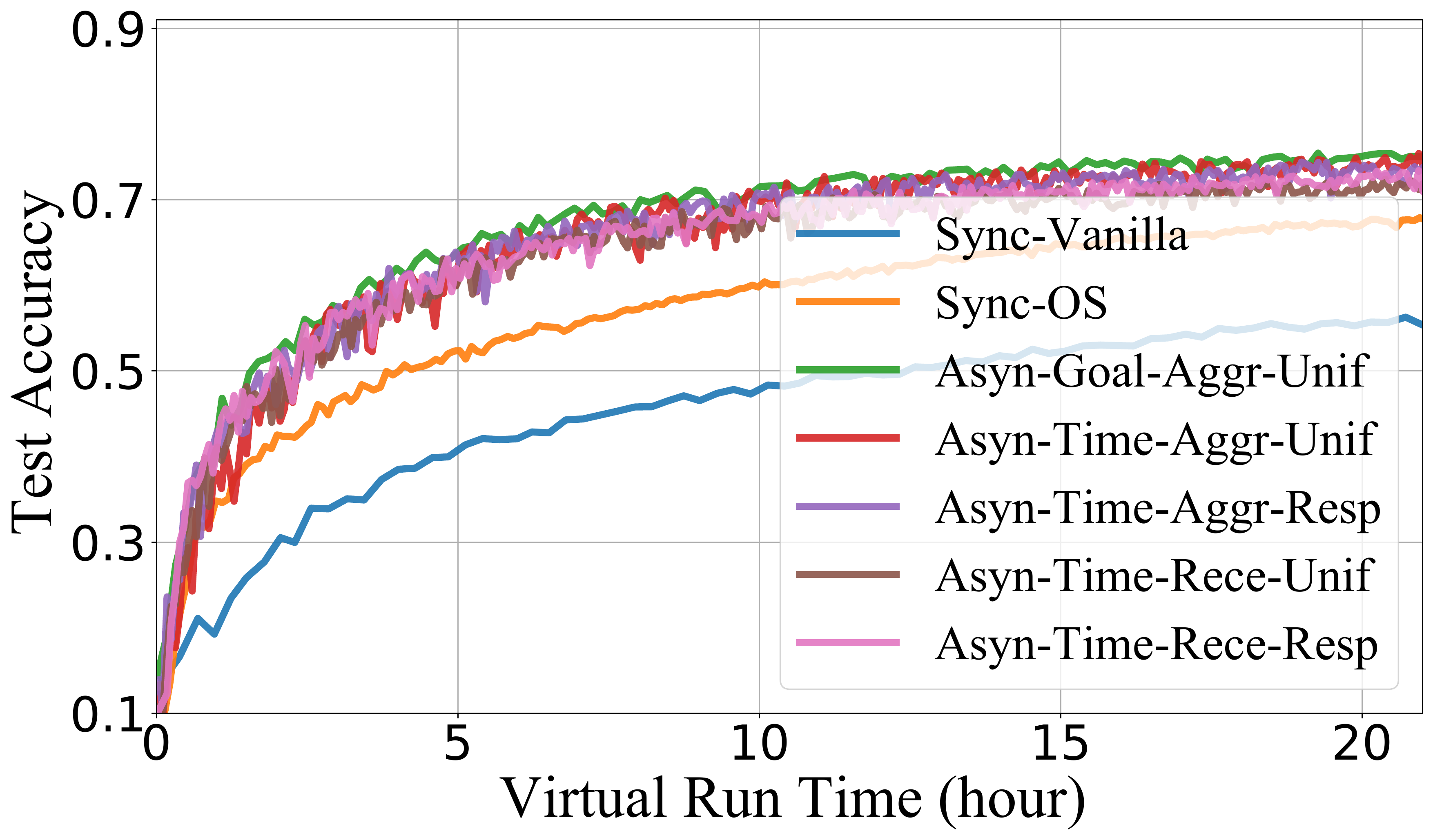}
    \vspace{-0.05in}
    \end{minipage}
    \caption{The experimental results of different asynchronous strategies provided in \ours.\label{fig:revision_asyn_results}}
\end{figure*}

\begin{figure*}
    \begin{minipage}[t]{0.33\textwidth}
    \setcaptionwidth{0.96\textwidth}
    \centering
    \includegraphics[width=0.96\textwidth]{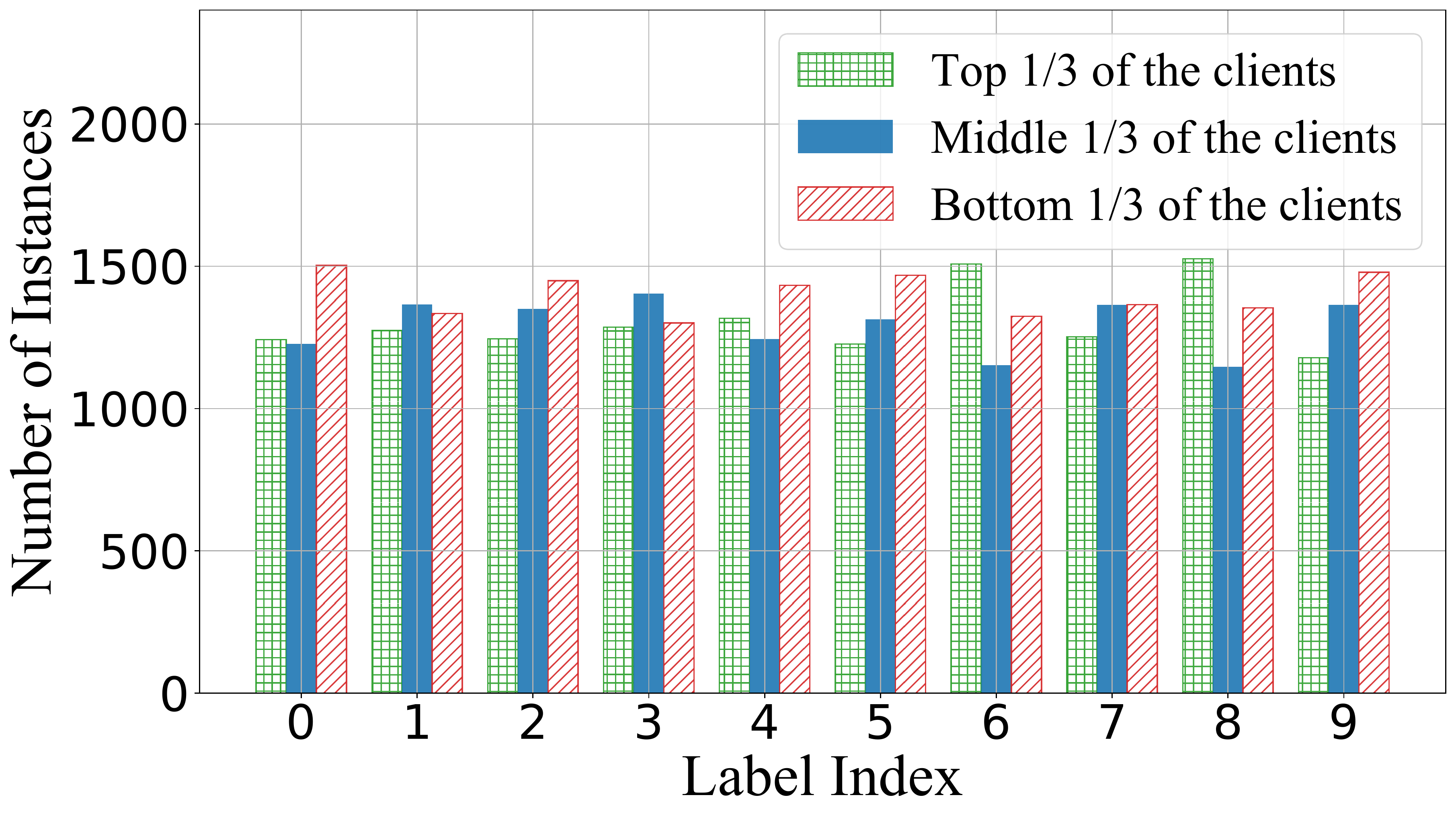}
    \vspace{-0.05in}
    \caption{The data distributions among different clustered clients (top clients response faster than bottom clients) in CIFAR-10.}
    \label{fig:label_distribution}
    \end{minipage}
     \begin{minipage}[t]{0.33\textwidth}
    \setcaptionwidth{0.96\textwidth}
    \centering
    \includegraphics[width=0.96\textwidth]{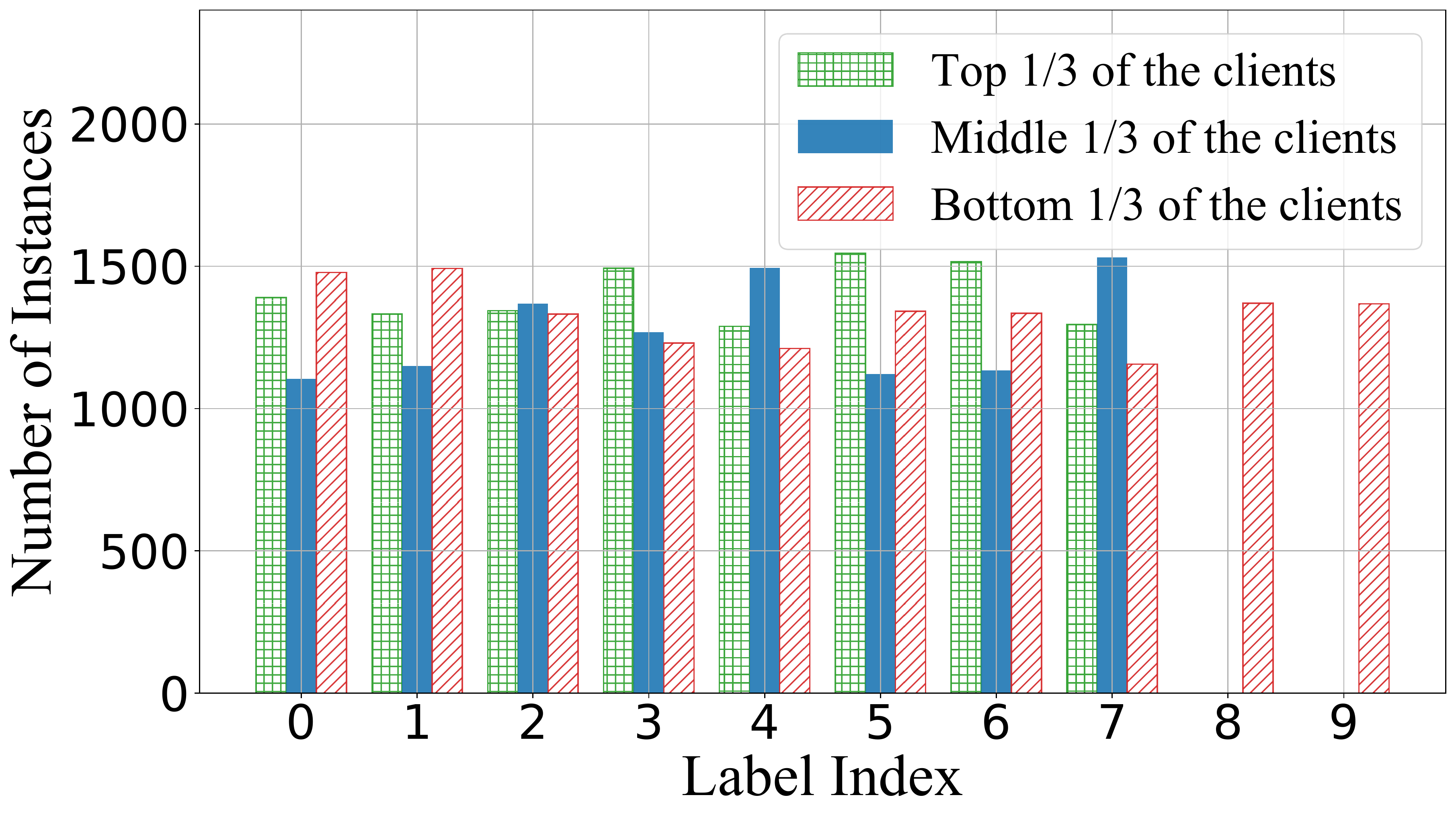}
    \vspace{-0.05in}
    \caption{The data distributions among different clustered clients (top clients response faster than bottom clients) in bias-CIFAR-10.}
    \label{fig:label_distribution_bias}
    \end{minipage}
     \begin{minipage}[t]{0.33\textwidth}
     \setcaptionwidth{0.96\textwidth}
    \centering
    \includegraphics[width=0.96\textwidth]{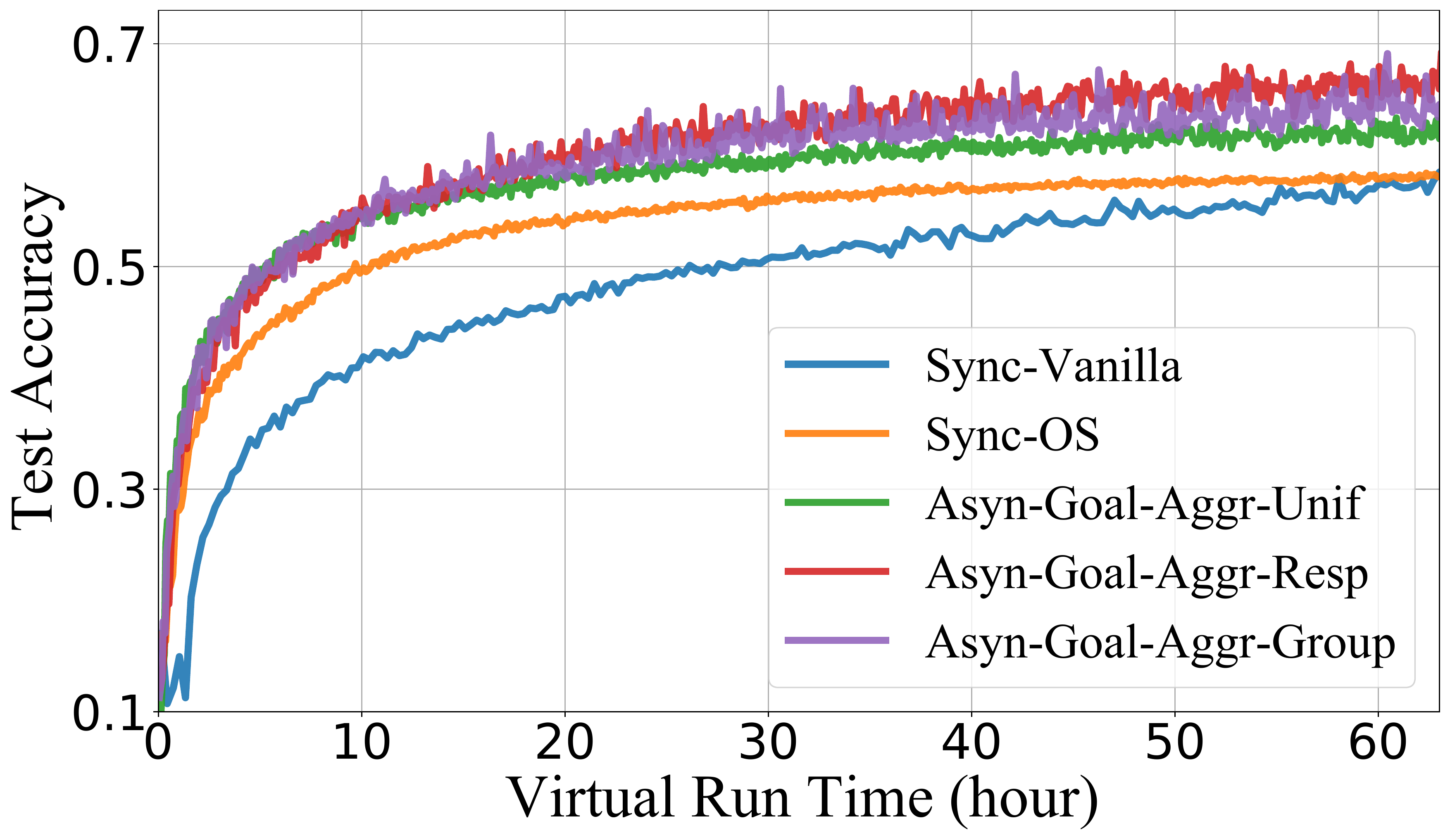}
    \vspace{-0.05in}
    \caption{The experimental results on bias-CIFAR-10.}
    \label{fig:asyn_bias}
    \end{minipage}
\end{figure*}

\section{example of amount of programming}
\label{appendix:programming_amount}
For the users who construct FL courses based on the built-in functionalities, they can adopt provided data, models, and  plug-in operations via simply configuring, or using their own data or models without changing the implementation of the federated behaviors in the provided FL courses. For example, excluding the codes for dataset preprocessing, users only need to modify two lines of code for adopting a customized dataset: one line for registering the dataset class and one line for changing the dataset name in the configuration. 

For the users who aim to do further customization, they can inherit from clients or servers (e.g., from the vanilla FedAvg) and add new $<$event, handler$>$ pairs to express the new behaviors of servers and clients accordingly. For example, to implement the personalized federated algorithms FedBN and Ditto based on the vanilla FedAvg, users only need to add/modify 19 and around 100 lines of code respectively. 

\section{experiments and analysis on asynchronous training strategies}
\label{appendix:more_asyn_results}

We add more experimental results of the asynchronous training strategies provided in \ours, as shown in Figure~\ref{fig:revision_asyn_results}. From the figures we can observe that,  asynchronous training strategies can achieve better performance compared to synchronous training strategies.

Note that different asynchronous training strategies might have different characteristics. For example, {\it after aggregating} broadcasting manner causes less staleness than {\it after receiving} (as shown in Figure~\ref{fig:staleness}), but {\it after aggregating} requires more available bandwidths at the server since multiple messages are sent out at the same time. These different characteristics can guide users to choose the more suitable asynchronous training strategies for their own application, considering model performance, system resource, staleness toleration, etc.

Generally speaking, there is no conclusion on which provided sampling strategy is better than another one, and the effectiveness of the sampling strategies is case dependent. Instead of providing the ``best'' asynchronous training strategies (in fact it is impossible because of ``No Free Lunch''), \ours aim to provide a good abstraction and modularization of asynchronous federated learning, which can both cover most of the existing studies and provide flexibility for customization. Users can conveniently choose the  suitable asynchronous training strategies accordingly.

As shown in Figure~\ref{fig:revision_asyn_results}, from the experiments conducted on CIFAR-10, we find that the performance of applying response-related and group sampling strategies are similar to that of applying uniform strategy for this specific case.
It is not a surprising result, since both response-related and group sampling strategies are proposed to alleviate the model bias caused by the heterogeneity in participants' system resources, e.g., clients with weak responsiveness might contribute little to federated aggregation as their (staled) updates might be discounted or even dropped out. However, the distributions of client data are independent of their distributions of system resources, which causes that when clustering the clients according to their system resources, the data distributions among different clusters are similar, as shown in Figure~\ref{fig:label_distribution}. Such similarities among clusters limit the effectiveness of response-related and group sampling strategies compared to the uniform sampling strategies, since there might not exist model bias caused by the heterogeneity in participants' system resources.
 
Furthermore, we re-split CIFAR-10 dataset and make the distributions of clients' data related to their system resources: We randomly select some labels as ``rare'' labels, and instances with these ``rare'' labels are only owned by clients with weak responsiveness. The distributions among clients on this dataset (called bias-CIFAR-10) are shown in Figure~\ref{fig:label_distribution_bias}, and the experimental results are shown in Figure~\ref{fig:asyn_bias}. We can observe that applying response-related and group sampling strategies can achieve noticeable improvement compared to uniform sampling strategies, which empirically confirms our analysis above.